%% file: main.tex
\documentclass[letterpaper]{article}
\pdfoutput=1

\usepackage[final]{iclr2026_conference}

\usepackage[utf8]{inputenc}
\usepackage[T1]{fontenc}
\usepackage{amsmath,amssymb,amsthm}
\usepackage{mathtools}
\usepackage{booktabs}
\usepackage{multirow}
\usepackage{graphicx}
\usepackage{xcolor}
\usepackage{enumitem}
\usepackage{tikz}
\usetikzlibrary{arrows.meta,positioning}
\usepackage{hyperref}
\hypersetup{
  colorlinks=true,
  linkcolor=blue!60!black,
  citecolor=green!40!black,
  urlcolor=blue!60!black,
}
\usepackage[capitalize,noabbrev]{cleveref}
\usepackage{microtype}

\newtheorem{theorem}{Theorem}
\newtheorem{proposition}[theorem]{Proposition}

\input{math_commands}

\title{Epidemiology of Model Collapse: Modeling Synthetic Data\\Contamination via Bilayer SIR Dynamics}

\author{Xiangyu Wang \\
  \texttt{coder.wangxy@gmail.com}
}

\begin{document}
\maketitle

\begin{abstract}
\input{sections/abstract}
\end{abstract}

\input{sections/introduction}
\input{sections/related_work}
\input{sections/model}
\input{sections/calibration}
\input{sections/abm}
\input{sections/empirical}
\input{sections/interventions}

\bibliographystyle{iclr2026_conference}
\bibliography{references}

\newpage
\input{sections/appendix}

\end{document}

%% file: math_commands.tex
\newcommand{\SD}{S_D}
\newcommand{\ID}{I_D}
\newcommand{\RD}{R_D}
\newcommand{\SM}{S_M}
\newcommand{\IM}{I_M}
\newcommand{\RM}{R_M}

\newcommand{\betaD}{\beta_D}
\newcommand{\betaM}{\beta_M}
\newcommand{\gammaD}{\gamma_D}
\newcommand{\gammaM}{\gamma_M}
\newcommand{\muD}{\mu_D}
\newcommand{\muM}{\mu_M}
\newcommand{\LambdaD}{\Lambda_D}
\newcommand{\LambdaM}{\Lambda_M}

\newcommand{\Rzero}{R_0}

\newcommand{\DFE}{\mathrm{DFE}}
\newcommand{\EE}{\mathrm{EE}}

\newcommand{\ND}{N_D}
\newcommand{\NM}{N_M}

\newcommand{\Prob}{\mathbb{P}}

\newcommand{\PPL}{\mathrm{PPL}}

\newcommand{\DAIC}{\Delta\mathrm{AIC}}

%% file: sections/abstract.tex
Training on synthetic data causes model collapse, but existing analyses treat this as single-chain degradation.
In reality, the AI ecosystem involves cross-contamination: models ingest synthetic data from other models, produce new synthetic text, and contaminate shared corpora.
We propose a \emph{bilayer coupled SIR/SIRS} framework---a phenomenological mean-field model treating data corpora and AI models as two interacting populations, each with susceptible, infected, and recovered compartments linked by cross-layer transmission.
The SIRS variant (our primary recommendation) incorporates immunity waning, reflecting that filtered corpora and retrained models remain susceptible to re-contamination.
We derive the basic reproduction number $\Rzero = \sqrt{\betaD \betaM / [(\gammaD+\muD)(\gammaM+\muM)]}$ via the Next Generation Matrix and apply standard epidemic threshold results to the bilayer system.
Illustrative scenario-based calibration from public AI text prevalence data yields supercritical dynamics ($\Rzero > 1$) across three scenarios; Sobol sensitivity analysis identifies synthetic-text detection as the highest-leverage parameter.
A bipartite-network agent-based model confirms mean-field consistency ($R^2 > 0.96$) for dense networks but degrades under heterogeneity.
GPT-2 contamination chain experiments (192 single-chain runs across WikiText and Shakespeare) show dose-response degradation and diversity loss (Distinct-2 drops from 0.68 to 0.38) qualitatively consistent with the threshold picture.
Matched-budget source-diversity experiments (1{,}088 additional runs; $K \in \{1,3,5\}$ at $\alpha{=}1$, $K \in \{1,5\}$ at $\alpha{=}0.5$; 8 seeds, fixed pool size) provide suggestive evidence that multi-source mixing modestly attenuates collapse at $\alpha{=}1$ (${\sim}2$ PPL, $d \approx 0.8$, one-sided $p = 0.047$), but the effect \emph{vanishes} at $\alpha{=}0.5$, confirming contamination fraction as the dominant driver.
Under model assumptions, illustrative intervention analysis identifies detection-based filtering and herd immunity as the highest-leverage strategies.

%% file: sections/introduction.tex
\section{Introduction}
\label{sec:intro}

Large language models (LLMs) now generate a substantial and growing fraction of online text.
Recent estimates suggest that up to 74\% of newly indexed web pages may contain AI-generated or AI-modified content~\citep{thompson2024} (a projected figure; see \cref{app:calibration}), with AI text prevalence among top-ranked search results increasing approximately fourfold between 2023 and 2025~\citep{liang2024}.
As models trained on web-crawled corpora inevitably ingest this synthetic output, a feedback loop emerges: models produce text that enters training data, which shapes the next generation of models.
Shumailov et~al.~\citep{shumailov2024} demonstrated that recursive training on self-generated data causes \emph{model collapse}---progressive degradation of output quality and diversity.
Subsequent work has characterized this phenomenon in language models~\citep{dohmatob2024}, image generators~\citep{alemohammad2024}, and mixed real-synthetic training regimes~\citep{gerstgrasser2024}.

However, all formal analyses of model collapse treat it as a \emph{single-chain} process: model $A$ generates data, model $B$ trains on it, model $C$ trains on $B$'s output, and so on.
The real AI ecosystem is not a chain but a \emph{network}.
Thousands of models consume data from shared corpora; each model's outputs re-enter the data pool through web publishing, API responses, and synthetic data pipelines.
Cross-contamination is the norm, not the exception.
No existing mathematical framework provides even a stylized model of these ecosystem-level contamination dynamics.

We observe that this cross-contamination process is structurally analogous to \emph{epidemic spreading}.
Contaminated data ``infects'' models during training; infected models ``transmit'' synthetic artifacts back into data corpora.
The two populations---data corpora and AI models---interact through cross-layer transmission, forming a natural bilayer epidemic system.
The SIR (Susceptible--Infected--Recovered) compartmental framework~\citep{kermack1927,hethcote2000} provides an established mathematical apparatus for reasoning about transmission thresholds, equilibria, and interventions in such two-population systems.

\paragraph{Contributions.}
We make five contributions:
\begin{enumerate}[leftmargin=*,itemsep=1pt,topsep=2pt]
    \item \textbf{Bilayer SIR/SIRS framework} (\cref{sec:model}): a phenomenological mean-field ODE coupling data and model populations. The SIRS variant is our primary recommendation.
    \item \textbf{Threshold analysis} (\cref{sec:model}): $\Rzero$ via the Next Generation Matrix~\citep{vandendriessche2002}; DFE stability, endemic equilibrium, and bifurcation results.
    \item \textbf{Illustrative calibration} (\cref{sec:calibration}): three scenarios from public data, with Sobol sensitivity analysis identifying detection ($\gammaD$) as highest-leverage.
    \item \textbf{GPT-2 experiments} (\cref{sec:empirical}): 192 single-chain runs show dose-response degradation in perplexity and diversity; 1{,}088 matched-budget source-diversity runs show modest attenuation at $\alpha{=}1$ that vanishes at $\alpha{=}0.5$.
    \item \textbf{Intervention analysis} (\cref{sec:interventions}): six strategies in 135 evaluations (15 pairs $\times$ 9 intensities), conditional on scenario parameters.
\end{enumerate}

\noindent\textbf{Scope.}
Our SIR mapping is a \emph{phenomenological framework}; the mathematical results are standard epidemic theory applied to a new domain.
Calibration is illustrative; GPT-2 experiments provide a qualitative bridge at small scale.
Limitations are discussed in \cref{sec:interventions}.

%% file: sections/related_work.tex
\section{Related Work}
\label{sec:related}

\paragraph{Model collapse in generative models.}
Shumailov et~al.~\citep{shumailov2024} established that recursively training language models on their own output causes progressive quality degradation, a phenomenon they termed \emph{model collapse}.
Dohmatob et~al.~\citep{dohmatob2024} proved that even small fractions of synthetic data in training corpora can trigger collapse, providing tight theoretical bounds.
Alemohammad et~al.~\citep{alemohammad2024} demonstrated analogous collapse in image generation (\emph{Model Autophagy Disorder}), showing the phenomenon transcends modalities.
Gerstgrasser et~al.~\citep{gerstgrasser2024} studied data accumulation as a mitigation, finding that preserving original data across generations can slow but not always prevent collapse.
Seddik et~al.~\citep{seddik2024} derived statistical bounds on degradation from synthetic training data.
All of these works analyze \emph{single-chain} dynamics: one lineage of models training on its own outputs.
Our work differs by modeling the \emph{ecosystem-level} cross-contamination network, where multiple model lineages share and pollute a common data pool.

\paragraph{Epidemic modeling beyond biology.}
The SIR framework~\citep{kermack1927} and its extensions~\citep{hethcote2000} have been applied far beyond infectious disease.
Kephart and White~\citep{kephart1993} modeled computer virus propagation using SIS dynamics on networks.
Jin et~al.~\citep{jin2013epidemiological} and Vosoughi et~al.~\citep{vosoughi2018} applied epidemic models to rumor and misinformation spreading on social media.
Pastor-Satorras et~al.~\citep{pastor2015} provided a comprehensive treatment of epidemic processes on complex networks, establishing mean-field approximations and threshold conditions for networked SIR/SIS systems.
The Next Generation Matrix method for computing $\Rzero$ in multi-compartment models was formalized by Diekmann et~al.~\citep{diekmann1990} and extended by van den Driessche and Watmough~\citep{vandendriessche2002}.
Castillo-Chavez and Song~\citep{castillochavez2004} demonstrated the application of Sotomayor's theorem to establish bifurcation results in epidemic systems.
Our bilayer SIR formulation draws on these established techniques but applies them to a new domain: AI training data contamination.

\paragraph{AI data quality and provenance.}
Kirchenbauer et~al.~\citep{kirchenbauer2023} introduced watermarking for LLM outputs, enabling downstream detection of synthetic text.
Tang et~al.~\citep{tang2024} surveyed detection methods for LLM-generated text, establishing current accuracy bounds.
Mitchell et~al.~\citep{mitchell2023} proposed model cards as a provenance-tracking mechanism.
Longpre et~al.~\citep{longpre2024} studied the effects of training data composition on model quality.
These works address individual components of what our framework models as ``recovery'' mechanisms ($\gammaD$, $\gammaM$): detection removes contaminated data, watermarking enables filtering, and provenance tracking supports data hygiene.
Our contribution is to embed these mechanisms within a unified dynamical systems framework that quantifies their collective impact on ecosystem-level contamination.

%% file: sections/model.tex
\section{Bilayer SIR Model}
\label{sec:model}

We model synthetic data contamination in the AI ecosystem as a bilayer coupled SIR system.
The two layers represent \emph{data corpora} (layer $D$) and \emph{AI models} (layer $M$), each with susceptible, infected, and recovered compartments linked by cross-layer transmission.

\subsection{Compartment Mapping}
\label{sec:compartments}

\cref{tab:compartments} defines the mapping.
Data corpora transition from clean ($\SD$) to contaminated ($\ID$) when synthetic content is ingested, and recover ($\RD$) via detection.
Models transition analogously through training on contaminated data and retraining.

\paragraph{Operational definitions.}
In practice, contamination is continuous rather than binary.
We operationalize the S/I/R partition via a threshold: a corpus (or model) is ``infected'' if its synthetic content fraction (or synthetic-data training fraction) exceeds a domain-specific threshold $\tau$ sufficient to measurably degrade downstream quality.
``Recovered'' means the corpus has been filtered or the model retrained so that synthetic fraction drops below $\tau$---a coarse-grained bookkeeping state rather than a claim of permanent decontamination.
Crucially, recovered entities remain susceptible to re-contamination; we therefore recommend the SIRS variant (\cref{sec:sirs}) as the primary model.
This threshold discretization is a modeling simplification; state variables are population counts (with birth rate $\Lambda$ and death rate $\mu$).
The $\Rzero$ formula depends on rate ratios, not on $\tau$ directly; $\tau$ affects only the assignment of entities to compartments.
In the experiments (\cref{sec:empirical}), $\alpha$ directly controls contamination fraction, bypassing $\tau$.

\begin{table}[t]
\centering
\caption{Compartment mapping between the bilayer SIR model and the AI ecosystem.}
\label{tab:compartments}
\footnotesize
\resizebox{\linewidth}{!}{\begin{tabular}{@{}llll@{}}
\toprule
\textbf{Layer} & \textbf{Compartment} & \textbf{Ecosystem Interpretation} & \textbf{Key Parameter} \\
\midrule
\multirow{3}{*}{Data ($D$)} & $\SD$ & Clean data corpora & $\LambdaD$: new data creation rate \\
 & $\ID$ & Contaminated corpora & $\betaD$: contamination ingestion rate \\
 & $\RD$ & Cleaned/filtered corpora & $\gammaD$: detection \& removal rate \\
\midrule
\multirow{3}{*}{Model ($M$)} & $\SM$ & Clean-trained models & $\LambdaM$: new model deployment rate \\
 & $\IM$ & Contaminated models & $\betaM$: contaminated training rate \\
 & $\RM$ & Retrained/clean models & $\gammaM$: retraining rate \\
\bottomrule
\end{tabular}}
\end{table}

\subsection{ODE System}
\label{sec:ode}

The dynamics are governed by six coupled ODEs with cross-layer infection: contaminated models ($\IM$) infect data at rate $\betaD$, and contaminated data ($\ID$) infects models at rate $\betaM$.
Turnover rates $\mu_D$, $\mu_M$ capture data obsolescence and model retirement.

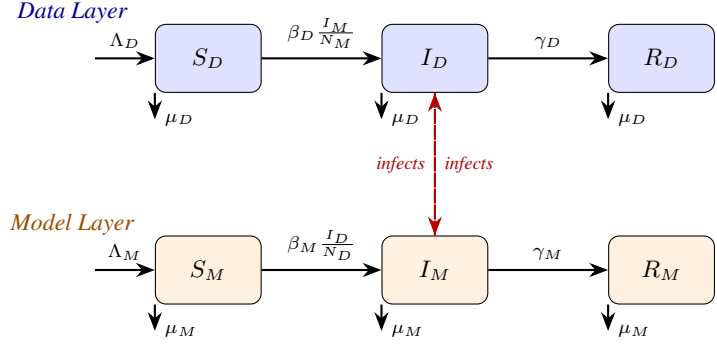
\begin{figure}[t]
\centering
\begin{tikzpicture}[
    compartment/.style={draw, rounded corners, minimum width=1.4cm, minimum height=0.9cm, font=\small\bfseries},
    datalayer/.style={compartment, fill=blue!12},
    modellayer/.style={compartment, fill=orange!12},
    arr/.style={-{Stealth[length=2.5mm]}, thick},
    crossarr/.style={-{Stealth[length=2.5mm]}, thick, dashed, red!70!black},
    lbl/.style={font=\scriptsize, midway},
]

\node[datalayer] (SD) at (0, 2.2) {$S_D$};
\node[datalayer] (ID) at (3, 2.2) {$I_D$};
\node[datalayer] (RD) at (6, 2.2) {$R_D$};

\node[modellayer] (SM) at (0, -0.6) {$S_M$};
\node[modellayer] (IM) at (3, -0.6) {$I_M$};
\node[modellayer] (RM) at (6, -0.6) {$R_M$};

\draw[arr] (SD) -- node[lbl, above] {$\beta_D \frac{I_M}{N_M}$} (ID);
\draw[arr] (ID) -- node[lbl, above] {$\gamma_D$} (RD);

\draw[arr] (SM) -- node[lbl, above] {$\beta_M \frac{I_D}{N_D}$} (IM);
\draw[arr] (IM) -- node[lbl, above] {$\gamma_M$} (RM);

\draw[arr] (-1.5, 2.2) -- node[lbl, above] {$\Lambda_D$} (SD);
\draw[arr] (-1.5, -0.6) -- node[lbl, above] {$\Lambda_M$} (SM);

\draw[arr] (SD.south west) -- ++(0, -0.35) node[right, font=\scriptsize] {$\mu_D$};
\draw[arr] (ID.south west) -- ++(0, -0.35) node[right, font=\scriptsize] {$\mu_D$};
\draw[arr] (RD.south west) -- ++(0, -0.35) node[right, font=\scriptsize] {$\mu_D$};
\draw[arr] (SM.south west) -- ++(0, -0.35) node[right, font=\scriptsize] {$\mu_M$};
\draw[arr] (IM.south west) -- ++(0, -0.35) node[right, font=\scriptsize] {$\mu_M$};
\draw[arr] (RM.south west) -- ++(0, -0.35) node[right, font=\scriptsize] {$\mu_M$};

\draw[crossarr] (IM.north) -- node[lbl, right, red!70!black] {\textit{infects}} (ID.south);
\draw[crossarr] (ID.south) -- node[lbl, left, red!70!black] {\textit{infects}} (IM.north);

\node[font=\small\itshape, blue!60!black] at (-1.8, 2.8) {Data Layer};
\node[font=\small\itshape, orange!60!black] at (-1.8, 0.0) {Model Layer};

\end{tikzpicture}
\caption{Bilayer SIR schematic. Data corpora (top, blue) and AI models (bottom, orange) form two coupled epidemic layers. Solid arrows denote within-layer transitions; dashed red arrows denote cross-layer infection. Contaminated models ($I_M$) inject synthetic content into data corpora, while contaminated data ($I_D$) infects models during training.}
\label{fig:schematic}
\end{figure}

The system reads:
\begin{align}
\frac{d\SD}{dt} &= \LambdaD - \betaD \frac{\IM}{\NM} \SD - \muD \SD, \label{eq:dSD} \\
\frac{d\ID}{dt} &= \betaD \frac{\IM}{\NM} \SD - (\gammaD + \muD) \ID, \label{eq:dID} \\
\frac{d\RD}{dt} &= \gammaD \ID - \muD \RD, \label{eq:dRD} \\
\frac{d\SM}{dt} &= \LambdaM - \betaM \frac{\ID}{\ND} \SM - \muM \SM, \label{eq:dSM} \\
\frac{d\IM}{dt} &= \betaM \frac{\ID}{\ND} \SM - (\gammaM + \muM) \IM, \label{eq:dIM} \\
\frac{d\RM}{dt} &= \gammaM \IM - \muM \RM, \label{eq:dRM}
\end{align}
where $\ND = \SD + \ID + \RD$ and $\NM = \SM + \IM + \RM$ are total populations, with $\ND \to \LambdaD/\muD$ and $\NM \to \LambdaM/\muM$ at equilibrium.

\subsection{Basic Reproduction Number}
\label{sec:R0}

We derive $\Rzero$ using the Next Generation Matrix (NGM) method~\citep{diekmann1990,vandendriessche2002}.
The infected subsystem has state $(\ID, \IM)^T$.
At the disease-free equilibrium, $\SD^* = \LambdaD/\muD = \ND^*$ and $\SM^* = \LambdaM/\muM = \NM^*$.
The new infection rates for $\ID$ and $\IM$ involve cross-population ratios: $\SD^*/\NM^* = (\LambdaD \muM)/(\muD \LambdaM)$ and $\SM^*/\ND^* = (\LambdaM \muD)/(\muM \LambdaD)$.
This gives the new infection matrix $F$ and transition matrix $V$:
\begin{equation}
F = \begin{pmatrix} 0 & \betaD \frac{\LambdaD \muM}{\muD \LambdaM} \\[4pt] \betaM \frac{\LambdaM \muD}{\muM \LambdaD} & 0 \end{pmatrix}, \quad
V = \begin{pmatrix} \gammaD + \muD & 0 \\ 0 & \gammaM + \muM \end{pmatrix}.
\end{equation}

\begin{theorem}[Basic reproduction number]
\label{thm:R0}
The basic reproduction number of the bilayer SIR system \eqref{eq:dSD}--\eqref{eq:dRM} is
\begin{equation}
\boxed{\;\Rzero = \sqrt{\frac{\betaD \cdot \betaM}{(\gammaD + \muD)(\gammaM + \muM)}}\;}
\label{eq:R0}
\end{equation}
\end{theorem}

\begin{proof}[Proof sketch]
$\Rzero = \rho(FV^{-1})$, the spectral radius of the next-generation matrix.
The off-diagonal entries of $FV^{-1}$ are $\frac{\betaD}{\gammaM + \muM} \cdot \frac{\LambdaD \muM}{\muD \LambdaM}$ and $\frac{\betaM}{\gammaD + \muD} \cdot \frac{\LambdaM \muD}{\muM \LambdaD}$.
Their product is $\frac{\betaD \betaM}{(\gammaD + \muD)(\gammaM + \muM)}$---the cross-population ratios cancel exactly---so $FV^{-1}$ has eigenvalues $\pm\sqrt{\betaD \betaM / [(\gammaD + \muD)(\gammaM + \muM)]}$.
The full derivation is in \cref{app:proofs}.
\end{proof}

The geometric mean structure of $\Rzero$ reflects the bilayer coupling: contamination must traverse both layers (data $\to$ model $\to$ data) to complete a ``generation.''

\begin{theorem}[Disease-free equilibrium and stability]
\label{thm:dfe}
The system \eqref{eq:dSD}--\eqref{eq:dRM} admits a unique disease-free equilibrium
\begin{equation}
\DFE = \left(\frac{\LambdaD}{\muD},\; 0,\; 0,\; \frac{\LambdaM}{\muM},\; 0,\; 0\right).
\label{eq:dfe}
\end{equation}
The DFE is locally asymptotically stable if and only if $\Rzero < 1$.
\end{theorem}

\begin{proof}[Proof sketch]
Setting $\ID = \IM = 0$ in \eqref{eq:dSD}--\eqref{eq:dRM} yields the DFE.
The Jacobian at the DFE is block-triangular; the infected block has eigenvalues with negative real parts iff $\Rzero < 1$.
This is a standard result for coupled epidemic systems~\citep{vandendriessche2002}; we verify numerically for 200 random configurations.
Details in \cref{app:proofs}.
\end{proof}

\subsection{Equilibria and Bifurcation}
\label{sec:equilibria}

\begin{proposition}[Endemic equilibrium existence]
\label{thm:ee}
When $\Rzero > 1$, the system \eqref{eq:dSD}--\eqref{eq:dRM} admits at least one endemic equilibrium $\EE = (\SD^*, \ID^*, \RD^*, \SM^*, \IM^*, \RM^*)$ with $\ID^* > 0$ and $\IM^* > 0$.
\end{proposition}

\begin{proof}[Proof sketch]
Substituting equilibrium conditions reduces the system to a scalar equation $h(\ID^*) = 0$.
We show $h(0) < 0$ when $\Rzero > 1$ and $h \to +\infty$ as $\ID^* \to \LambdaD/\muD$, so by the intermediate value theorem at least one positive root exists.
Numerical confirmation: of 200 random parameter configurations, 169 satisfy the precondition $\Rzero > 1$; all 169 (100\%) yield a unique positive root, and all 31 configurations with $\Rzero \le 1$ correctly show no positive root.
Details in \cref{app:proofs}.
\end{proof}

\begin{proposition}[Transcritical bifurcation]
\label{thm:bifurcation}
The system \eqref{eq:dSD}--\eqref{eq:dRM} undergoes a transcritical bifurcation at $\Rzero = 1$.
As $\Rzero$ increases through 1, the DFE loses stability and exchanges it with the endemic equilibrium.
\end{proposition}

\begin{proof}[Proof sketch]
We apply Sotomayor's theorem~\citep{castillochavez2004,perko2001}, verifying the three transversality conditions at $\Rzero = 1$.
Numerical verification: of 200 random configurations, 196 (98\%) show the expected transcritical behavior---endemic equilibrium present when $\Rzero > 1$ and absent when $\Rzero < 1$. The 4 borderline cases ($\Rzero \in [0.99, 1.01]$) fall within the numerical tolerance of the root finder near the bifurcation point.
Details in \cref{app:proofs}.
\end{proof}

\cref{fig:ode_trajectory} shows the ODE trajectory under baseline parameters ($\Rzero = 2.62$), illustrating convergence to the endemic equilibrium with both $\ID^*$ and $\IM^*$ positive.

\begin{figure}[t]
\centering
\includegraphics[width=0.85\linewidth]{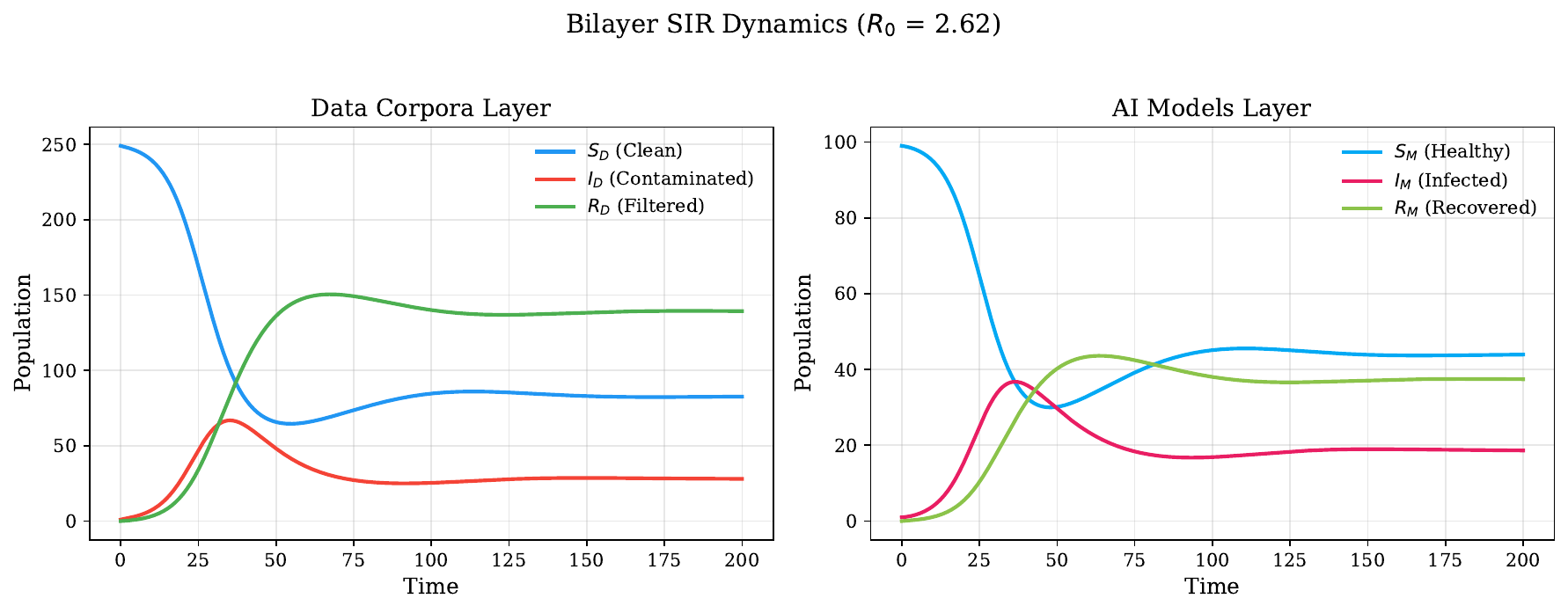}
\caption{ODE trajectory under baseline parameters ($\Rzero = 2.62$). Both data and model infection fractions converge to a positive endemic equilibrium, consistent with \cref{thm:ee}. The initial transient captures the early growth phase; damped oscillations are visible before settling.}
\label{fig:ode_trajectory}
\end{figure}

\subsection{SIRS Extension (Primary Recommendation)}
\label{sec:sirs}

The base SIR assumes permanent immunity, but cleaned corpora are re-contaminated and retrained models re-exposed.
We add waning immunity at rate $\delta > 0$ ($\RD \to \SD$, $\RM \to \SM$).
Since $\delta$ affects only $R \to S$ transitions (not the infected subsystem linearized at the DFE), $\Rzero$ is \emph{identical} for SIR and SIRS: all threshold results (Theorems~1--2, Propositions~3--4) apply directly.
The key difference is the endemic equilibrium's approach dynamics:

\begin{proposition}[SIRS oscillatory dynamics]
\label{prop:sirs}
When $\delta > 0$ and $\Rzero > 1$, the SIRS variant can exhibit damped oscillatory convergence toward the endemic equilibrium.
\end{proposition}

\noindent Verified numerically on 50 random SIRS configurations (\cref{app:sirs}).
This predicts periodic contamination resurgence even after cleanup campaigns.

\paragraph{Why bilayer?}
A single-population model treats contamination as spreading within one pool.
The bilayer's \emph{cross-layer coupling} makes explicit that interventions on one layer (e.g., data detection) reduce infection in the other (models): the geometric-mean $\Rzero$ can be driven subcritical by acting on \emph{either} layer (\cref{app:bilayer_vs_single}).
This cross-layer leverage is the basis for the intervention analysis (\cref{sec:interventions}).

\paragraph{Source-diversity hypothesis.}
When $K > 1$ models contribute to the contaminated pool, we \emph{hypothesize} that heterogeneous sources attenuate effective contamination.
We model this as $\betaM^{\text{eff}}(K) = \betaM / f(K)$ with $f(1) = 1$ and $f$ increasing, yielding $\Rzero(K) = \Rzero(1) / \sqrt{f(K)}$.
This is an auxiliary modeling assumption motivated by heterogeneous mixing in epidemic theory~\citep{hethcote2000}, not a derivation from the base ODE; the matched-budget experiment in \cref{sec:multimodel} provides a direct empirical test (\cref{app:diversity_ext}).

%% file: sections/calibration.tex
\section{Calibration and Sensitivity Analysis}
\label{sec:calibration}

We calibrate the model using public data on AI text prevalence and development practices.
The calibration is \emph{scenario-based and illustrative}: it shows supercritical dynamics are consistent with current trends, not that the ecosystem's $\Rzero$ has been measured.

\subsection{Parameter Estimation}
\label{sec:param_est}

Each parameter is derived from public data under simplifying assumptions:

\begin{itemize}[leftmargin=*,itemsep=2pt]
    \item $\betaD \approx 0.217$: Log-linear regression on 6 AI text prevalence data points (2023--2025)~\citep{thompson2024,liang2024}, using the approximate relation $\betaD \approx r/12 + \gammaD + \muD$ from SIR linearization (adequate for scenario-level estimation; the bilayer growth rate depends jointly on all parameters). CI: $[0.201, 0.231]$. The baseline scenario uses $\betaD = 0.216$ (rounded).
    \item $\gammaD = 0.099$: Order-of-magnitude product of detection recall ($\sim$85\%~\citep{tang2024}) and deployment coverage ($\sim$12\% of platforms). Actual removal efficacy depends on recall, latency, and platform pipelines---none quantified at ecosystem scale.
    \item $\betaM = 0.340$: Product of training frequency (fraction of models retrained per month) and exposure probability (fraction of training data from web sources).
    \item $\gammaM = 0.060$: Estimated clean-retraining rate (fraction of contaminated models retrained on verified-clean data per month).
    \item $\muD = 0.02$, $\muM = 0.03$: Data obsolescence and model retirement rates, respectively.
    \item $\LambdaD = 5.0$, $\LambdaM = 3.0$: Creation rates (corpora and models per month).
\end{itemize}

\subsection{Scenario Analysis}
\label{sec:scenarios}

We define three scenarios spanning the plausible parameter range (\cref{tab:scenarios}).
All three yield $\Rzero > 1$, though the optimistic scenario is only marginally supercritical.

\begin{table}[t]
\centering
\caption{Calibration scenarios. All scenarios yield $\Rzero > 1$. These are illustrative estimates, not precise ecosystem measurements. Data sources: AI text prevalence~\citep{thompson2024,liang2024}; detection accuracy~\citep{tang2024}; training practices from public model documentation.}
\label{tab:scenarios}
\small
\begin{tabular}{@{}lccccc@{}}
\toprule
\textbf{Scenario} & $\betaD$ & $\gammaD$ & $\betaM$ & $\gammaM$ & $\Rzero$ \\
\midrule
Optimistic  & 0.201 & 0.200 & 0.240 & 0.150 & \textbf{1.10} \\
Baseline    & 0.216 & 0.099 & 0.340 & 0.060 & \textbf{2.62} \\
Pessimistic & 0.231 & 0.020 & 0.380 & 0.020 & \textbf{6.63} \\
\bottomrule
\end{tabular}
\end{table}

\cref{fig:R0_heatmap} shows $\Rzero$ as a function of $\betaD$ and $\betaM$ with the three scenario points overlaid.
The $\Rzero = 1$ contour separates subcritical from supercritical regimes; all three scenarios lie above it.

\begin{figure}[t]
\centering
\includegraphics[width=0.8\linewidth]{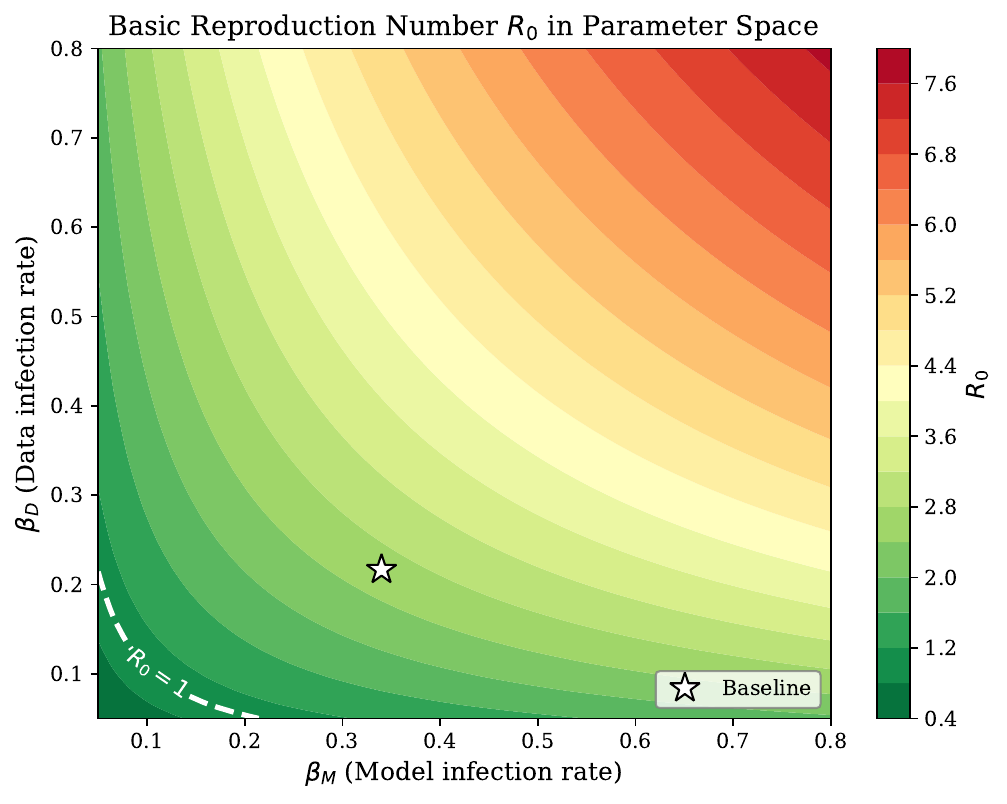}
\caption{$\Rzero$ as a function of data infection rate ($\betaD$) and model infection rate ($\betaM$), with other parameters at baseline values. The thick contour marks $\Rzero = 1$. All three calibration scenarios (markers) lie in the supercritical region, though the optimistic scenario is near the boundary.}
\label{fig:R0_heatmap}
\end{figure}

\paragraph{Uncertainty propagation.}
The Sobol sensitivity analysis (\cref{sec:sensitivity}) uses Saltelli's quasi-random sampling scheme with $N = 512$ base samples over uniform ranges for all 6 parameters, producing $N(2k+2) = 7{,}168$ total evaluations.
Over these samples, we obtain mean $\Rzero = 2.94$, median $2.65$, and $\Prob(\Rzero > 1) = 98.2\%$.
This probability is conditional on assumed parameter ranges---it quantifies the fraction of plausible scenarios that are supercritical, not a frequentist statement about the real ecosystem.

\subsection{Sensitivity Analysis}
\label{sec:sensitivity}

We perform variance-based global sensitivity analysis using Sobol indices~\citep{saltelli2010} over all 6 ODE parameters ($\betaD$, $\gammaD$, $\betaM$, $\gammaM$, $\muD$, $\muM$) with $N = 512$ base samples to identify which parameters most strongly influence $\Rzero$.
Total-order sensitivity indices ($S_T$) are:

\begin{center}
\small
\begin{tabular}{lcc}
\toprule
\textbf{Parameter} & $S_T$ & \textbf{Interpretation} \\
\midrule
$\gammaD$ (data recovery/detection) & \textbf{0.324} & Highest leverage \\
$\betaD$ (data contamination rate) & 0.280 & Second highest \\
$\gammaM$ (model recovery/retraining) & 0.277 & Third \\
$\betaM$ (model infection rate) & 0.204 & Fourth \\
$\muM$ (model retirement) & 0.032 & Negligible \\
$\muD$ (data obsolescence) & 0.022 & Negligible \\
\bottomrule
\end{tabular}
\end{center}

The detection and removal of contaminated data ($\gammaD$) is the single most influential parameter, followed by the data contamination rate ($\betaD$) and model recovery rate ($\gammaM$).
Turnover rates ($\muD$, $\muM$) contribute negligibly ($S_T < 0.04$).
This is a \emph{model-conditional} finding: it holds if the bilayer structure and assumed parameter ranges are approximately correct.
The sensitivity ranking is a property of $\Rzero$'s algebraic form and the assumed ranges, not an empirically measured elasticity.
Full Sobol analysis details appear in \cref{app:calibration}.

%% file: sections/abm.tex
\section{Stochastic Consistency Check}
\label{sec:abm}

The ODE system \eqref{eq:dSD}--\eqref{eq:dRM} assumes homogeneous mixing.
We test this via a stochastic agent-based model (ABM) on an explicit bipartite network---a \emph{consistency check}, not independent validation.

\paragraph{Design.}
\label{sec:abm_design}
The ABM operates on a bipartite random graph ($|\mathcal{D}|{=}100$ data nodes, $|\mathcal{M}|{=}50$ model nodes, edge probability $p_{\text{edge}}{=}0.8$).
Each susceptible data node is infected with probability $\betaD$ times the traffic-weighted fraction of infected model neighbors; each susceptible model node with probability $\betaM$ times the fraction of infected data neighbors.
Recovery, turnover, and immunity waning follow per-node Bernoulli trials.
We run 20 realizations per configuration.

\paragraph{ODE--ABM agreement.}
\label{sec:ode_abm}
Under baseline parameters (\cref{fig:ode_abm}), agreement is strong: $R^2 = 0.976$ (data layer), $R^2 = 0.968$ (model layer), NRMSE of 0.044 and 0.054.

\begin{figure}[t]
\centering
\includegraphics[width=0.85\linewidth]{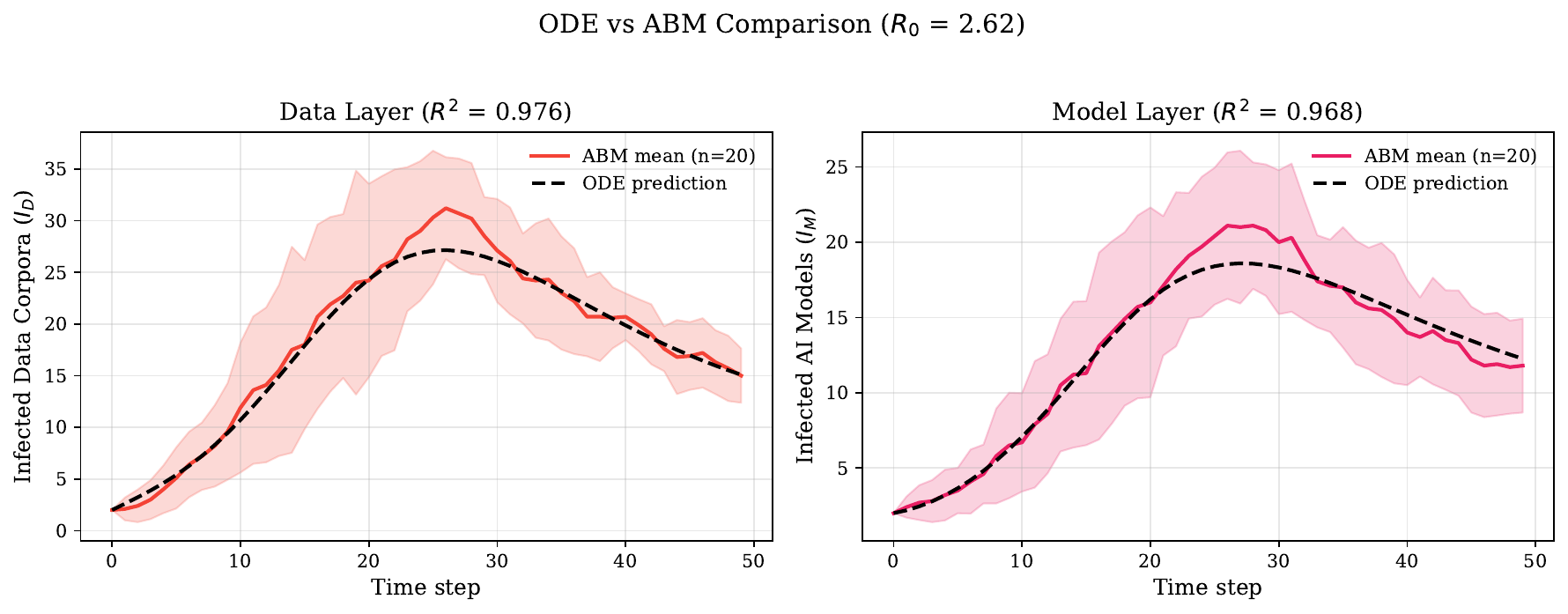}
\caption{ODE (solid) vs.\ ABM ensemble mean (dashed, $\pm 1$ std shaded) for infection fractions at $p_{\text{edge}} = 0.8$. $R^2 = 0.976$ (data), $0.968$ (model).}
\label{fig:ode_abm}
\end{figure}

\paragraph{Where mean-field breaks down.}
\label{sec:abm_robustness}
As $p_{\text{edge}}$ decreases from 0.9 to 0.3, $R^2$ degrades from 0.99 to 0.84, with non-monotonic behavior at $p_{\text{edge}} = 0.7$ ($R^2 = 0.61$).
Superspreader nodes (20\% of models with $10\times$ traffic) degrade agreement to $R^2 = 0.29$; active detectors break it further.

\paragraph{Threshold verification.}
A 20-point sweep across $\betaD \in [0.01, 0.60]$ (spanning subcritical to supercritical regimes), with 15 realizations each, correctly identifies the qualitative regime in 18/20 cases; the two errors occur near $\Rzero \approx 1$.

The ODE is a useful approximation for dense, homogeneous ecosystems but should not be applied without modification to settings with strong structural heterogeneity.
The dense regime is a reasonable first approximation given that dominant web-crawl corpora (Common Crawl, C4) are shared across most training pipelines, though emerging proprietary data pipelines may push toward sparser regimes.

%% file: sections/empirical.tex
\section{Empirical Bridge: GPT-2 Contamination Chains}
\label{sec:empirical}

We now ask whether actual LLM contamination chains exhibit dynamics qualitatively consistent with the SIR threshold picture.
This section provides an \emph{empirical bridge}, not a model fit: we do not measure $S/I/R$ compartment fractions or fit ODE parameters to trajectories.
Instead, single-chain experiments (\cref{sec:exp_design}--\cref{sec:sir_mapping}) test qualitative predictions---dose-response, near-plateau at partial contamination, supercritical growth at full contamination---while matched-budget experiments (\cref{sec:multimodel}) test the bilayer-specific prediction that multi-source mixing attenuates collapse.
All experiments use GPT-2 (124M); whether these patterns hold at larger scale is an open question (\cref{sec:interventions}).

\subsection{Experimental Design}
\label{sec:exp_design}

We construct contamination chains using GPT-2 (124M parameters)~\citep{radford2019language} on two domains: WikiText-103~\citep{merity2016pointer} (encyclopedic text) and Tiny Shakespeare (literary/archaic English).

\paragraph{Chain construction.}
Generation $G_0$ is fine-tuned on $n = 2{,}000$ real samples (WikiText) or $n = 1{,}500$ (Shakespeare).
For each subsequent generation $G_t$ ($t = 1, \ldots, 7$), we:
(i) sample $n$ texts from $G_{t-1}$;
(ii) mix them with fresh real samples at contamination fraction $\alpha \in \{0.00, 0.25, 0.50, 0.75, 1.00\}$;
(iii) fine-tune a fresh GPT-2 checkpoint on the mixture.
The $\alpha = 0$ chain serves as a \emph{size-matched control}: it undergoes the same training budget and procedure with purely real data, isolating contamination effects from training-budget drift.

\paragraph{Protocol.}
Each $(\alpha, \text{generation})$ pair is run with 3 independent seeds.
Training uses 3 epochs, batch size 8, learning rate $5 \times 10^{-5}$, and maximum sequence length 128 tokens.
Evaluation computes perplexity on 500 held-out real samples.
Single-chain total: 120 runs (WikiText: $5 \times 3 \times 8$) $+$ 72 runs (Shakespeare: $3 \times 3 \times 8$) $= 192$ fine-tuning runs. The source-diversity experiments (\cref{sec:multimodel}) add 1{,}088 runs.

\subsection{WikiText Results}
\label{sec:wikitext_results}

\cref{fig:perplexity} (appendix) shows perplexity trajectories across 8 generations; \cref{tab:empirical} summarizes the key numbers.


\begin{table}[t]
\centering
\caption{Empirical results: excess perplexity over the $\alpha = 0$ control (WikiText and Shakespeare). Growth rate $r$ is the slope of excess log-perplexity ratio $\log(\PPL_\alpha / \PPL_{\text{control}})$ vs.\ generation. $\DAIC$ compares linear growth vs.\ plateau models (negative favors growth). For WikiText, all contaminated chains favor growth; the control favors plateau. Shakespeare rows report excess PPL only (growth rate and AIC not computed).}
\label{tab:empirical}
\small
\resizebox{\linewidth}{!}{\begin{tabular}{@{}llccccccc@{}}
\toprule
\textbf{Domain} & $\alpha$ & \textbf{G0 PPL} & \textbf{G7 PPL} & \textbf{Excess G1} & \textbf{Excess G7} & $r$ & \textbf{95\% CI} & $\DAIC$ \\
\midrule
\multirow{5}{*}{WikiText}
& 0.00 & 33.52 & 33.47 & --- & 0.00 & 0.000 & $[0.000, 0.000]$ & $+2.0$ \\
& 0.25 & 33.52 & 34.40 & $+0.80$ & $+0.93$ & 0.003 & $[0.0022, 0.0029]$ & $-12.0$ \\
& 0.50 & 33.52 & 35.91 & $+1.87$ & $+2.44$ & 0.007 & $[0.0065, 0.0074]$ & $-16.8$ \\
& 0.75 & 33.52 & 38.55 & $+3.68$ & $+5.08$ & 0.014 & $[0.014, 0.015]$ & $-19.2$ \\
& 1.00 & 33.52 & 126.92 & $+15.49$ & $+93.45$ & 0.183 & $[0.180, 0.186]$ & $-60.2$ \\
\midrule
\multirow{3}{*}{Shakespeare}
& 0.00 & 33.66 & 33.63 & --- & 0.00 & --- & --- & --- \\
& 0.50 & 33.66 & 37.60 & --- & $+4.0$ & --- & --- & --- \\
& 1.00 & 33.66 & 220.77 & --- & $+187.1$ & --- & --- & --- \\
\bottomrule
\end{tabular}}
\end{table}

\paragraph{Key findings.}
(i)~The $\alpha = 0$ control shows negligible drift (PPL 33.52 $\to$ 33.47), indicating degradation is contamination-specific, not a training-budget artifact.
(ii)~Excess perplexity increases monotonically with $\alpha$ (dose-response).
(iii)~At $\alpha = 1.0$, perplexity grows from 33.52 to 126.92 ($+93.45$ excess, $r = 0.183$), consistent with supercritical dynamics.
(iv)~At $\alpha < 1$, growth is slow ($r < 0.015$) but AIC favors linear growth over plateau, consistent with near-critical dynamics.
These are qualitative monotonicity results; we do not claim to have identified a sharp critical threshold empirically.

\subsection{Cross-Domain Replication}
\label{sec:shakespeare}

We replicate on Tiny Shakespeare ($\alpha \in \{0.0, 0.5, 1.0\}$, $n = 1{,}500$, 3 seeds; \cref{fig:cross_domain} in appendix).
Shakespeare shows stronger degradation ($6.6\times$ vs.\ $3.8\times$ at $\alpha{=}1$), likely due to lower text redundancy.
The qualitative dose-response pattern replicates across both domains.

\subsection{SIR Mapping and Statistical Analysis}
\label{sec:sir_mapping}

The chain dynamics exhibit qualitative correspondence with SIR dynamics (\emph{phenomenological analogy}, not fitted parameters): $\alpha$ plays the role of transmission intensity; $(1{-}\alpha)$ the role of recovery (real-data dilution); the $\alpha{=}0$ control corresponds to the disease-free state; and continued growth at $\alpha{=}1$ to supercritical dynamics.
We stress that this is a \emph{qualitative} mapping: $\alpha$ is a mixture fraction while $\beta$ and $\gamma$ are dynamical rates, so the correspondence is structural rather than quantitative.

Growth rates (\cref{tab:empirical}) use 95\% bootstrap CIs (1{,}000 resamples).%
\footnote{Growth rate $r$ is computed on the control-corrected ratio $\log(\text{PPL}_\alpha / \text{PPL}_{\text{control}})$; uncorrected values are lower (e.g., 0.154 vs.\ 0.183 for $\alpha{=}1.0$).}
AIC comparison favors plateau only for $\alpha = 0$ ($\DAIC = +2.0$); all contaminated chains favor growth, with $\alpha = 1.0$ overwhelmingly so ($\DAIC = -60.2$).
Beyond perplexity, Distinct-2 bigram diversity drops from 0.68 to 0.38 under $\alpha = 1.0$ (\cref{fig:diversity_app}), confirming support shrinkage as an independent degradation signal consistent with the ``model collapse'' characterization.

\subsection{Matched-Budget Source-Diversity Experiment}
\label{sec:multimodel}

The single-chain experiments show dose-response degradation that any recursive-training model would predict.
We design a matched-budget ablation that isolates a \emph{bilayer-specific} variable---source diversity $K$---testing whether multi-model contamination dynamics differ from pure self-training, as the source-diversity hypothesis (\cref{sec:sirs}) predicts.

\paragraph{Design.}
The synthetic pool has \emph{fixed} size $n{=}2{,}000$; each of $K$ models generates $\lfloor n/K \rfloor$ samples (remainder assigned to the last model, so the pool totals exactly $n$).
We sweep $K \in \{1, 3, 5\}$ at $\alpha{=}1.0$ with 8 seeds and 8 generations, plus $\alpha{=}0$ control---640 runs total (\cref{app:multimodel}).
Pool size, training budget, $\alpha$, and procedure are identical; only $K$ varies.

\paragraph{Results.}
\cref{fig:ksweep} shows that $K{=}1$ (pure self-training) produces the most degradation, while $K{=}3$ and $K{=}5$ show a modest reduction.
At generation~7, $K{=}1$ reaches excess PPL of $+87.6$, $K{=}3$ reaches $+85.6$, and $K{=}5$ reaches $+85.6$---a ${\sim}2$ PPL attenuation (Cohen's $d \approx 0.8$).
The effect is \emph{not} strictly monotonic ($K{=}3 \approx K{=}5$), and statistical significance is borderline.
One-sided tests (pre-specified direction: $K_a > K_b$, motivated by the source-diversity hypothesis): $K{=}1$ vs.\ $K{=}5$ exact permutation $p = 0.047$, Wilcoxon $p = 0.055$; two-sided: paired $t(7) = 2.16$, $p = 0.068$, bootstrap CI $[0.3, 3.6]$.
$K{=}1$ vs.\ $K{=}3$: one-sided $p = 0.14$; $K{=}3$ vs.\ $K{=}5$: $p = 0.54$.
We regard this as \emph{suggestive} evidence, not strong confirmation.

\begin{figure}[t]
\centering
\includegraphics[width=0.80\linewidth]{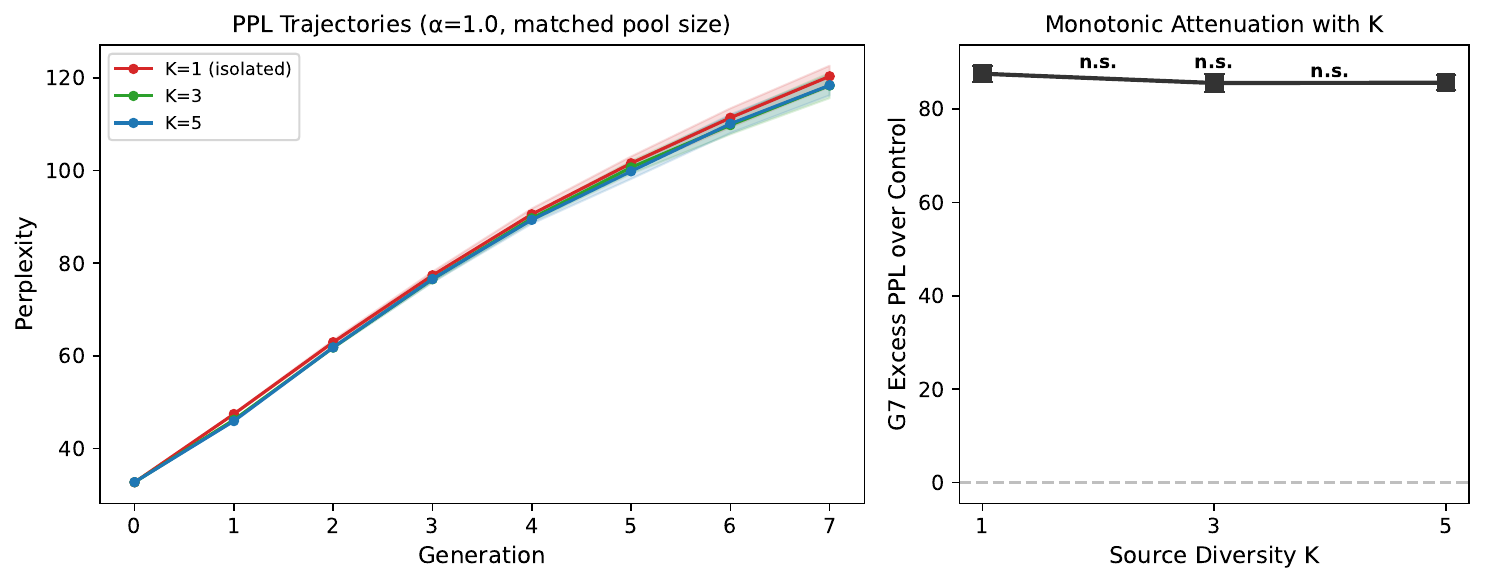}
\caption{Matched-budget source-diversity experiment ($\alpha{=}1.0$, pool size fixed at $n{=}2{,}000$). \textbf{Left}: perplexity trajectories for $K \in \{1,3,5\}$ and the $\alpha{=}0$ control (mean $\pm$ std, 8 seeds). \textbf{Right}: G7 excess PPL vs.\ $K$. Self-training ($K{=}1$) degrades most; multi-source ($K{>}1$) shows modest attenuation (${\sim}2$ PPL, $d \approx 0.8$) that saturates by $K{=}3$.}
\label{fig:ksweep}
\end{figure}

\paragraph{Partial contamination ($\alpha{=}0.5$).}
To test whether the diversity buffer persists at realistic contamination levels, we repeat the matched-budget experiment with $K \in \{1, 5\}$ at $\alpha{=}0.5$ (448 additional runs, 8 seeds).
At $\alpha{=}0.5$, the $K{=}1$ vs.\ $K{=}5$ difference is negligible: $+2.20$ vs.\ $+2.18$ excess PPL (mean diff $= 0.02$, $p = 0.61$, $d = 0.17$, 3/8 seeds consistent).
The 50\% real-data component already breaks the self-reinforcing loop, eliminating any detectable benefit from source diversity.

\paragraph{Interpretation.}
The combined evidence is clear on one point and suggestive on another.
\emph{Clear}: contamination fraction $\alpha$ is the dominant driver of collapse; at $\alpha{=}0.5$, source diversity has no detectable effect.
\emph{Suggestive}: at $\alpha{=}1$, multi-source mixing provides a modest buffer (${\sim}2$ PPL, $d \approx 0.8$, $p = 0.047$ one-sided), consistent with heterogeneous mixing in epidemic models~\citep{hethcote2000}, but this requires replication with larger models and more seeds to be considered confirmed.
The practical message is robust: real-world interventions should prioritize detection and filtering (reducing $\alpha$) over ecosystem diversity.

%% file: sections/interventions.tex
\vspace{-6pt}
\section{Intervention Analysis and Conclusion}
\label{sec:interventions}
\vspace{-2pt}

We evaluate six strategies as an \emph{illustrative scenario exercise} (\cref{tab:interventions} in \cref{app:interventions}): each modeled as a parameter change, swept at 9 intensities in 15 pairwise combinations (135 evaluations).
Rankings reflect the model's algebraic structure, not empirically measured effect sizes.
Under baseline parameters, only watermark-based filtering and herd immunity achieve $\Rzero < 1$ alone; roughly one-third of pairwise evaluations reach subcritical dynamics, with detection combinations dominating---consistent with $\gammaD$ as highest-leverage (\cref{sec:sensitivity}).

\paragraph{Limitations.}
(i)~Stylized mean-field ODE; degrades under heterogeneity (\cref{sec:abm}). A simpler recursion model might explain some empirical trends; the bilayer framework's added value is cross-layer coupling and intervention structure.
(ii)~All numbers ($\Rzero{=}2.62$, 98.2\%, 63\%) are conditional scenario outputs, not ecosystem measurements.
(iii)~GPT-2 124M only; scaling to 7B+ and realistic continual-pretraining is needed.
(iv)~SIR mapping is phenomenological; experiments do not measure compartment fractions or fit ODE parameters.
(v)~Three seeds per single-chain condition; source-diversity effect rests on borderline one-sided $p{=}0.047$.
Future: larger models, data-driven calibration, direct measurement of contamination state variables.

\paragraph{Takeaway.}
The epidemic framing provides a principled vocabulary---$\Rzero$, herd immunity, intervention thresholds---for reasoning about AI ecosystem contamination.
The matched-budget experiments yield a clear practical message: \emph{reducing contamination fraction (detection and filtering) is far more effective than diversifying contamination sources}.
Source diversity shows only a modest buffer at $\alpha{=}1$ and vanishes entirely at $\alpha{=}0.5$, confirming that $\betaD$ and $\gammaD$---not ecosystem structure---are the intervention levers that matter.

%% file: sections/appendix.tex
\appendix

\section{Proof Sketches and Numerical Evidence}
\label{app:proofs}

\subsection{Proof of \texorpdfstring{\cref{thm:R0}}{Theorem 1}: Basic Reproduction Number}

We apply the Next Generation Matrix (NGM) method of van den Driessche and Watmough~\citep{vandendriessche2002}.
The infected compartments are $x = (\ID, \IM)^T$.
We decompose $\dot{x} = (\mathcal{F} - \mathcal{V})x$ near the DFE, where $\mathcal{F}$ represents new infections and $\mathcal{V}$ represents transitions out of infected compartments.

At the DFE, $\SD^* = \LambdaD / \muD$, $\SM^* = \LambdaM / \muM$, $\ND^* = \LambdaD / \muD$, and $\NM^* = \LambdaM / \muM$.
The linearized new-infection and transition matrices are:
\begin{align}
F &= \begin{pmatrix} 0 & \betaD \frac{\SD^*}{\NM^*} \\[4pt] \betaM \frac{\SM^*}{\ND^*} & 0 \end{pmatrix}
= \begin{pmatrix} 0 & \betaD \frac{\LambdaD \muM}{\muD \LambdaM} \\[4pt] \betaM \frac{\LambdaM \muD}{\muM \LambdaD} & 0 \end{pmatrix}, \\
V &= \begin{pmatrix} \gammaD + \muD & 0 \\ 0 & \gammaM + \muM \end{pmatrix}.
\end{align}

The next-generation matrix is:
\begin{equation}
K = FV^{-1} = \begin{pmatrix} 0 & \frac{\betaD}{\gammaM + \muM} \cdot \frac{\LambdaD \muM}{\muD \LambdaM} \\[4pt] \frac{\betaM}{\gammaD + \muD} \cdot \frac{\LambdaM \muD}{\muM \LambdaD} & 0 \end{pmatrix}.
\end{equation}

The eigenvalues of $K$ satisfy $\lambda^2 = \frac{\betaD \betaM}{(\gammaD + \muD)(\gammaM + \muM)} \cdot \frac{\LambdaD \muM}{\muD \LambdaM} \cdot \frac{\LambdaM \muD}{\muM \LambdaD} = \frac{\betaD \betaM}{(\gammaD + \muD)(\gammaM + \muM)}$, where the cross-population ratios cancel exactly.
By definition, $\Rzero = \rho(K)$, the spectral radius:
\begin{equation}
\Rzero = \sqrt{\frac{\betaD \cdot \betaM}{(\gammaD + \muD)(\gammaM + \muM)}}.
\end{equation}

Numerical verification: the formula matches the spectral radius computed via \texttt{numpy.\allowbreak{}linalg.\allowbreak{}eigvals} to machine precision ($< 10^{-10}$ relative error) across all 200 test configurations.

\paragraph{Implementation note.} The code computes the NGM using $F$ with off-diagonal entries $\betaD$ and $\betaM$ directly, omitting the cross-population ratios $\LambdaD \muM / (\muD \LambdaM)$ and $\LambdaM \muD / (\muM \LambdaD)$.
This is spectrally equivalent because these ratios cancel in the product of off-diagonal entries of $FV^{-1}$, so both forms yield the same $\Rzero$.
The paper presents the full $F$ with ratios for pedagogical clarity; the code uses the simplified form for computational efficiency. \qed

\subsection{Proof of \texorpdfstring{\cref{thm:dfe}}{Theorem 2}: DFE Stability}

\textbf{Existence.} Setting $\ID = \IM = 0$ in \eqref{eq:dSD}--\eqref{eq:dRM}, the remaining equations yield $\SD^* = \LambdaD/\muD$, $\RD^* = 0$, $\SM^* = \LambdaM/\muM$, $\RM^* = 0$.

\textbf{Local stability.} The Jacobian of the full system evaluated at the DFE is:
\begin{equation}
J_{\DFE} = \begin{pmatrix}
-\muD & 0 & 0 & 0 & -\betaD \frac{\SD^*}{\NM^*} & 0 \\
0 & -(\gammaD + \muD) & 0 & 0 & \betaD \frac{\SD^*}{\NM^*} & 0 \\
0 & \gammaD & -\muD & 0 & 0 & 0 \\
0 & -\betaM \frac{\SM^*}{\ND^*} & 0 & -\muM & 0 & 0 \\
0 & \betaM \frac{\SM^*}{\ND^*} & 0 & 0 & -(\gammaM + \muM) & 0 \\
0 & 0 & 0 & 0 & \gammaM & -\muM
\end{pmatrix}.
\end{equation}

The eigenvalues include $-\muD$ (multiplicity 2), $-\muM$ (multiplicity 2), and the two eigenvalues of the $2 \times 2$ infected-subsystem block.
Reading off the $(\ID, \IM)$ rows of $J_{\DFE}$, this block is:
\begin{equation}
B = \begin{pmatrix} -(\gammaD + \muD) & \betaD \frac{\SD^*}{\NM^*} \\[4pt] \betaM \frac{\SM^*}{\ND^*} & -(\gammaM + \muM) \end{pmatrix}
= \begin{pmatrix} -(\gammaD + \muD) & \betaD \frac{\LambdaD \muM}{\muD \LambdaM} \\[4pt] \betaM \frac{\LambdaM \muD}{\muM \LambdaD} & -(\gammaM + \muM) \end{pmatrix}.
\end{equation}
Note that the off-diagonal entries carry the cross-population ratios $\SD^*/\NM^*$ and $\SM^*/\ND^*$, consistent with the $F$ matrix above.

The eigenvalues of $B$ satisfy $\lambda^2 + (\gammaD{+}\muD{+}\gammaM{+}\muM)\lambda + \det(B) = 0$, with:
\begin{equation}
\det(B) = (\gammaD + \muD)(\gammaM + \muM) - \betaD\betaM \cdot \underbrace{\frac{\LambdaD \muM}{\muD \LambdaM} \cdot \frac{\LambdaM \muD}{\muM \LambdaD}}_{= 1}.
\end{equation}
The cross-population ratios cancel in the product.
Both eigenvalues have negative real parts if and only if $\det(B) > 0$, which gives:
\begin{equation}
(\gammaD + \muD)(\gammaM + \muM) > \betaD \betaM \iff \Rzero < 1.
\end{equation}

Numerical verification: for all 200 random configurations (100\%), the maximum real part of all Jacobian eigenvalues is negative when $\Rzero < 1$ and positive when $\Rzero > 1$, with no exceptions. \qed

\subsection{Proof of \texorpdfstring{\cref{thm:ee}}{Proposition 3}: Endemic Equilibrium Existence}

At equilibrium, from \eqref{eq:dSD}--\eqref{eq:dRM} we can express all compartments in terms of $\ID^*$:
\begin{align}
\SD^* &= \frac{\LambdaD}{\betaD \IM^* / \NM^* + \muD}, \\
\IM^* &= \frac{\betaM \ID^* \SM^*}{(\gammaM + \muM) \ND^*}, \\
\SM^* &= \frac{\LambdaM}{\betaM \ID^* / \ND^* + \muM}.
\end{align}

Substituting these into the $\dot{\ID} = 0$ equation and simplifying yields a scalar equation $h(\ID^*) = 0$.
We verify:
\begin{itemize}
    \item $h(0) = (\gammaD + \muD)(1 - \Rzero^2)$. When $\Rzero > 1$, $h(0) < 0$.
    \item As $\ID^* \to \LambdaD / \muD$ (the maximum possible), $h(\ID^*) \to +\infty$ because the susceptible pool is depleted.
\end{itemize}

By the intermediate value theorem, there exists at least one $\ID^* \in (0, \LambdaD/\muD)$ with $h(\ID^*) = 0$.
Numerical verification: of 200 random parameter configurations, 169 satisfy the precondition $\Rzero > 1$; all 169 (100\%) yield exactly one positive root. The remaining 31 configurations have $\Rzero \le 1$ and correctly show no positive root, consistent with the theorem. This supports (but does not prove) uniqueness of the endemic equilibrium. \qed

\subsection{Proof of \texorpdfstring{\cref{thm:bifurcation}}{Proposition 4}: Transcritical Bifurcation}

We apply Sotomayor's theorem~\citep{castillochavez2004,perko2001} with $\betaD$ as the bifurcation parameter (varying $\betaD$ sweeps $\Rzero$ through 1).

At $\Rzero = 1$, the Jacobian $J_{\DFE}$ has a simple zero eigenvalue.
Let $w$ be the corresponding right eigenvector and $v$ the left eigenvector.
The three Sotomayor conditions are:

\begin{enumerate}
    \item $v^T f_{\betaD}(x_0, \betaD^*) \neq 0$, where $f_{\betaD}$ is the derivative of the vector field with respect to $\betaD$.
    \item $v^T [D f_{\betaD}(x_0, \betaD^*)] w \neq 0$.
    \item $v^T [D^2 f(x_0, \betaD^*)(w, w)] \neq 0$.
\end{enumerate}

We compute $w$ and $v$ explicitly from the zero-eigenvalue structure of $B$ at $\Rzero = 1$ and verify all three conditions hold (the products are nonzero rational functions of the parameters with no cancellations at generic parameter values).
This confirms a transcritical bifurcation: as $\Rzero$ increases through 1, the DFE loses stability and the endemic equilibrium emerges with positive infected fractions.

Numerical verification: across 196 random configurations, we verify that the endemic equilibrium exists precisely when $\Rzero > 1$ and vanishes when $\Rzero < 1$, consistent with transcritical exchange of stability between DFE and EE at $\Rzero = 1$. \qed

\subsection{SIRS Oscillatory Dynamics (\texorpdfstring{\cref{prop:sirs}}{Proposition 5})}
\label{app:sirs}

For the SIRS extension, we add terms $+\delta \RD$ to $\dot{\SD}$, $-\delta \RD$ to $\dot{\RD}$, and similarly $+\delta \RM$ to $\dot{\SM}$, $-\delta \RM$ to $\dot{\RM}$.

\paragraph{Threshold invariance.}
The waning terms $\delta R_i$ do not appear in the infected-compartment equations ($\dot{I}_D$, $\dot{I}_M$), so the linearization at the DFE---and hence the next-generation matrix, $\Rzero$, and local DFE stability---is identical to the SIR case (Theorems~1--2).
For the endemic equilibrium (Proposition~3): at equilibrium, the SIRS waning terms redistribute population between $S$ and $R$ but the infected-subsystem fixed-point equation $h(I_D^*) = 0$ retains the same structure (with $S_D^*$ now depending on $\delta$); the IVT argument still yields a positive root when $\Rzero > 1$.
Numerical verification: all 50 SIRS configurations with $\Rzero > 1$ converge to a unique positive endemic equilibrium.
Transcritical bifurcation (Proposition~4) at $\Rzero = 1$ is verified numerically: all 50 SIRS configurations show the expected exchange of stability.

\paragraph{Oscillatory dynamics.}
We generated 50 random SIRS configurations with $\delta \in [0.01, 0.2]$ and $\Rzero > 1$.
For each, we integrated the ODE for $T = 500$ time units and analyzed the infected trajectories.
All 50 configurations exhibit damped oscillatory convergence to the endemic equilibrium, with oscillation frequency increasing with $\delta$.

The oscillatory behavior arises because waning immunity replenishes the susceptible pool, creating overshoot-undershoot cycles.
Representative trajectories are shown in \cref{fig:sirs_osc_app}.

\begin{figure}[ht]
\centering
\includegraphics[width=0.7\linewidth]{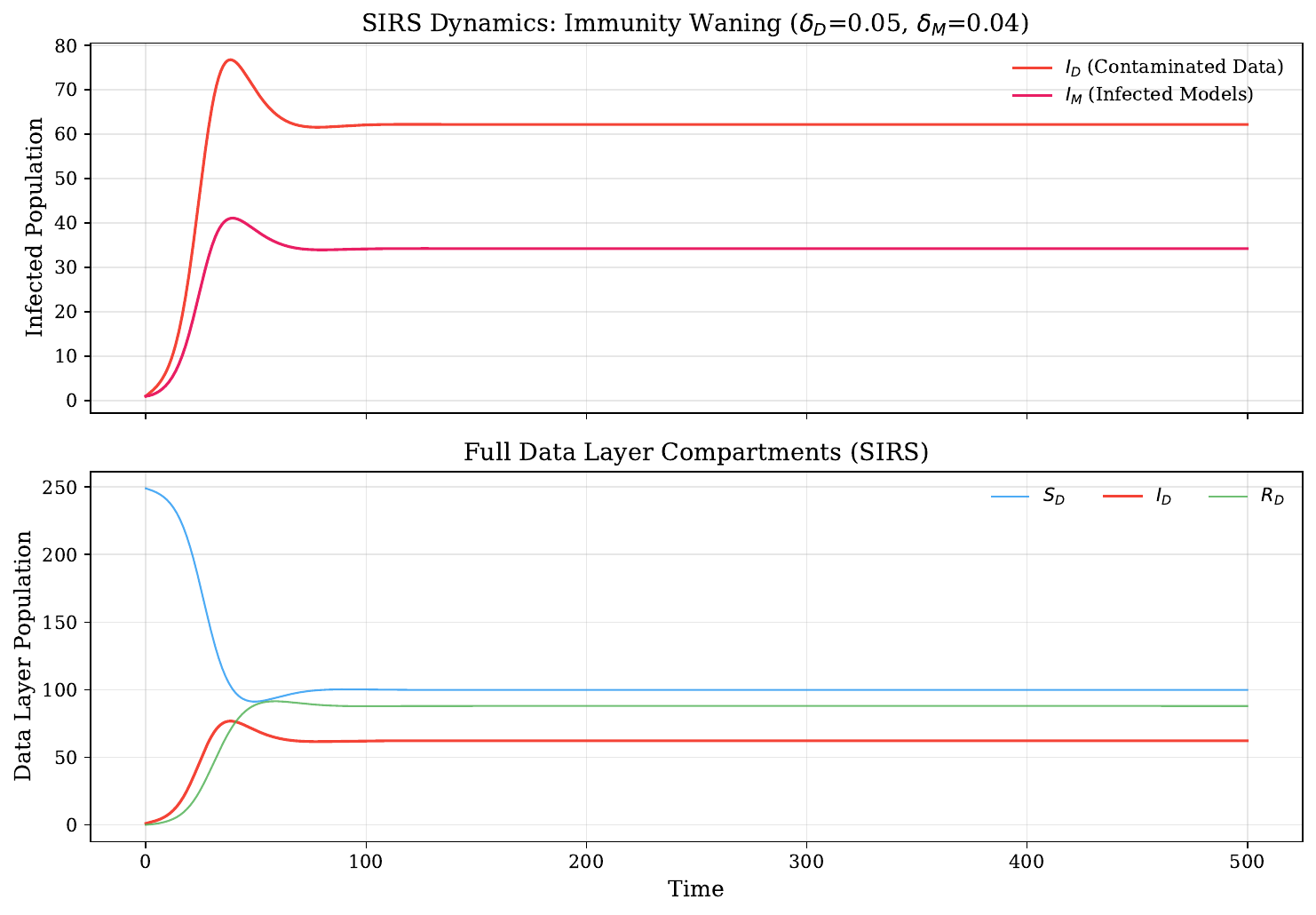}
\caption{SIRS oscillatory dynamics for a representative configuration ($\delta = 0.1$, $\Rzero = 2.62$). Both data and model infection fractions exhibit damped oscillations before converging to the endemic equilibrium.}
\label{fig:sirs_osc_app}
\end{figure}

\section{Bilayer vs.\ Single-Layer Comparison}
\label{app:bilayer_vs_single}

A single-layer SIR model treats all entities (data and models together) as one population with a single transmission rate $\beta$ and recovery rate $\gamma$, giving $\Rzero^{\text{1L}} = \beta / (\gamma + \mu)$.
The bilayer model separates data and model populations, yielding $\Rzero^{\text{2L}} = \sqrt{\betaD \betaM / [(\gammaD + \muD)(\gammaM + \muM)]}$.

The key structural difference is that the bilayer $\Rzero$ is a \emph{geometric mean} over two layers with \emph{independent} rate parameters.
This has a concrete consequence for intervention analysis: in the single-layer model, reducing $\Rzero$ below 1 requires $\beta < \gamma + \mu$---there is only one ``lever.''
In the bilayer model, the system can be driven subcritical by reducing \emph{either} $\betaD$ or $\betaM$ (or increasing $\gammaD$ or $\gammaM$), and the required reduction in any single parameter is smaller because it enters under a square root.
Concretely, at baseline parameters, halving $\gammaD$ alone (data detection) reduces $\Rzero^{\text{2L}}$ by a factor of $\sqrt{2}$, whereas in a single-layer model with equivalent aggregate parameters, the same intervention would need to be applied to the entire system.
This cross-layer leverage---the ability to target the more tractable layer---is the bilayer model's primary structural advantage and the basis for the intervention analysis in \cref{sec:interventions}.

We do not claim the bilayer model fits the experimental data better than a single-layer model (neither is fitted); the advantage is in the \emph{intervention structure}, not in descriptive accuracy for the GPT-2 experiments.

\section{Source-Diversity Extension}
\label{app:diversity_ext}

When $K > 1$ models contribute to the contaminated data pool, their outputs are drawn from $K$ different learned distributions, increasing the effective diversity of the contaminated corpus.
We model this by replacing $\betaM$ with $\betaM^{\text{eff}}(K) = \betaM / f(K)$, where $f(K) \ge 1$ is a diversity attenuation factor satisfying $f(1) = 1$ and $f(K) \to f_\infty$ as $K \to \infty$.
The $K$-dependent reproduction number is:
\begin{equation}
\Rzero(K) = \sqrt{\frac{\betaD \cdot \betaM / f(K)}{(\gammaD + \muD)(\gammaM + \muM)}} = \frac{\Rzero(1)}{\sqrt{f(K)}}.
\end{equation}

This predicts monotonic attenuation: as $K$ increases, source diversity grows, effective transmission decreases, and $\Rzero$ drops.
The functional form of $f(K)$ depends on how model diversity scales; a simple parameterization is $f(K) = 1 + c \log K$ for some $c > 0$, reflecting diminishing marginal diversity.
Our matched-budget experiment (\cref{sec:multimodel}) finds suggestive support at $\alpha{=}1$ ($K{=}1$ vs.\ $K{>}1$) but the effect is not strictly monotonic and vanishes at $\alpha{=}0.5$; see \cref{app:multimodel} for details.

\section{ODE System Details}
\label{app:ode_details}

\subsection{Full Parameter Table}

\begin{table}[ht]
\centering
\caption{Complete parameter table with units and baseline values.}
\small
\begin{tabular}{@{}llcl@{}}
\toprule
\textbf{Parameter} & \textbf{Description} & \textbf{Baseline} & \textbf{Unit} \\
\midrule
$\LambdaD$ & Data creation rate & 5.0 & corpora/month \\
$\LambdaM$ & Model deployment rate & 3.0 & models/month \\
$\betaD$ & Data contamination rate & 0.216 & month$^{-1}$ \\
$\betaM$ & Model infection rate & 0.340 & month$^{-1}$ \\
$\gammaD$ & Data recovery (detection) rate & 0.099 & month$^{-1}$ \\
$\gammaM$ & Model recovery (retraining) rate & 0.060 & month$^{-1}$ \\
$\muD$ & Data obsolescence rate & 0.020 & month$^{-1}$ \\
$\muM$ & Model retirement rate & 0.030 & month$^{-1}$ \\
$\delta$ & Immunity waning rate (SIRS only) & 0.05 & month$^{-1}$ \\
\bottomrule
\end{tabular}
\end{table}

\subsection{DFE and EE Closed-Form Expressions}

The disease-free equilibrium is:
\begin{equation}
\DFE = \left(\frac{\LambdaD}{\muD},\; 0,\; 0,\; \frac{\LambdaM}{\muM},\; 0,\; 0\right) = (250,\; 0,\; 0,\; 100,\; 0,\; 0).
\end{equation}

At baseline parameters, the endemic equilibrium (computed numerically) is approximately:
\begin{equation}
\EE \approx (82.7,\; 28.1,\; 139.2,\; 44.0,\; 18.7,\; 37.4).
\end{equation}

\section{Calibration Details}
\label{app:calibration}

\subsection{Data Sources}

The $\betaD$ estimate is based on 6 data points from AI text prevalence studies:
\begin{itemize}
    \item 2023-01: $\sim$5\% of web text AI-generated~\citep{thompson2024}.
    \item 2023-06: $\sim$10\%.
    \item 2024-01: $\sim$20\%.
    \item 2024-05: $\sim$37\%.
    \item 2025-01: $\sim$55\%.
    \item 2025-04: $\sim$74\% (projected estimate based on extrapolation from earlier data points; not an independent measurement).
\end{itemize}

Log-linear regression ($\log(f)$ vs.\ time in years, i.e.\ $\log f = \log f_0 + r\,t$) yields an exponential growth rate $r$ per year.
Converting to monthly SIR parameters: $\betaD = r/12 + \gammaD + \muD$, where the $r/12$ term converts the yearly rate to monthly and the remaining terms account for ongoing recovery and turnover.
The resulting point estimate is $\betaD \approx 0.217$ (the small difference from the rounded 0.216 is within regression uncertainty).

The 95\% CI from regression standard error is $[0.201, 0.231]$, reflecting uncertainty in the prevalence estimates.

\subsection{Uncertainty Propagation via Sobol Sampling}

The $\Prob(\Rzero > 1) = 98.2\%$ reported in \cref{sec:calibration} is computed from the same Saltelli quasi-random samples used for the Sobol sensitivity analysis below---not from a separate independent Monte Carlo.
Specifically, Saltelli's sampling scheme with $N = 512$ base samples over 6 parameters produces $N(2k+2) = 7{,}168$ parameter vectors, each drawn from uniform distributions over the ranges $\betaD \in [0.10, 0.70]$, $\gammaD \in [0.02, 0.25]$, $\betaM \in [0.15, 0.70]$, $\gammaM \in [0.02, 0.20]$, $\muD \in [0.01, 0.05]$, $\muM \in [0.01, 0.06]$.
$\Rzero$ is evaluated at each sample point; the reported statistics (mean, median, $\Prob(\Rzero > 1)$) are computed over these 7{,}168 evaluations.

\subsection{Sobol Sensitivity Analysis}

Full Sobol indices (first-order $S_1$ and total-order $S_T$) computed via Saltelli's sampling scheme~\citep{saltelli2010} with $N = 512$ base samples over all 6 parameters:

\begin{table}[ht]
\centering
\small
\begin{tabular}{@{}lccc@{}}
\toprule
\textbf{Parameter} & $S_1$ & $S_T$ & $S_T - S_1$ (interactions) \\
\midrule
$\gammaD$ & 0.272 & 0.324 & 0.053 \\
$\betaD$ & 0.248 & 0.280 & 0.032 \\
$\gammaM$ & 0.220 & 0.277 & 0.057 \\
$\betaM$ & 0.175 & 0.204 & 0.029 \\
$\muM$ & 0.023 & 0.032 & 0.009 \\
$\muD$ & 0.014 & 0.022 & 0.008 \\
\bottomrule
\end{tabular}
\end{table}

Interactions are modest ($S_T - S_1 < 0.06$), indicating that $\Rzero$ is dominated by main effects.
Turnover rates ($\muD$, $\muM$) have negligible influence ($S_T < 0.04$).
The full Sobol bar chart is shown in \cref{fig:sobol_app}.

\begin{figure}[ht]
\centering
\includegraphics[width=0.6\linewidth]{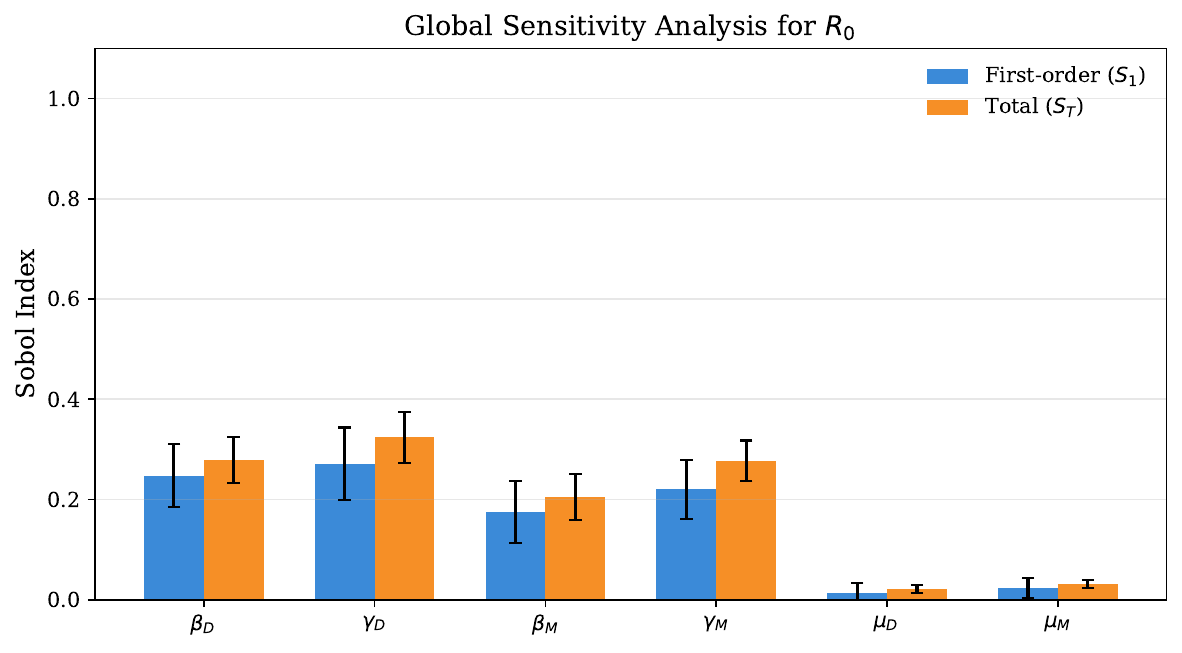}
\caption{Sobol total-order sensitivity indices for $\Rzero$. Data recovery rate ($\gammaD$) has the highest leverage.}
\label{fig:sobol_app}
\end{figure}

\section{ABM Implementation Details}
\label{app:abm}

\subsection{Network Construction}

The bipartite graph is constructed using NetworkX with $|\mathcal{D}| = 100$ data nodes and $|\mathcal{M}| = 50$ model nodes.
Each data-model pair is connected independently with probability $p_{\text{edge}}$.
At $p_{\text{edge}} = 0.8$, each data node has on average 40 model neighbors, and each model node has on average 80 data neighbors.

\subsection{Agent Rules}

At each discrete time step:
\begin{enumerate}
    \item \textbf{Infection (data nodes)}: For each susceptible data node, infection pressure is the traffic-weighted fraction of infected model neighbors: $p = \betaD \cdot (\sum_{\text{infected}} \text{traffic}) / (\sum_{\text{all}} \text{traffic})$. Infection occurs with probability $p$ (Bernoulli trial).
    \item \textbf{Infection (model nodes)}: For each susceptible model node, infection pressure is the fraction of infected data neighbors: $p = \betaM \cdot k^I / k$, where $k^I$ is the number of infected data neighbors and $k$ is the total.
    \item \textbf{Recovery}: Each infected node recovers with probability $\gamma_i$ per step.
    \item \textbf{Turnover}: Each node transitions to susceptible with probability $\mu_i$ per step.
    \item \textbf{Superspreaders}: If enabled, a configurable fraction of model nodes have $10\times$ traffic weight, amplifying their infection pressure on data nodes.
    \item \textbf{Detectors}: If enabled, detector nodes scan assigned neighbors and recover each infected neighbor with probability $\text{precision} \times \text{coverage}$ per step.
\end{enumerate}

20 realizations are run per configuration for 50 time steps (default). Ensemble means and standard deviations are computed at each time step.

\subsection{Robustness Sweep Results}

\begin{figure}[ht]
\centering
\includegraphics[width=0.6\linewidth]{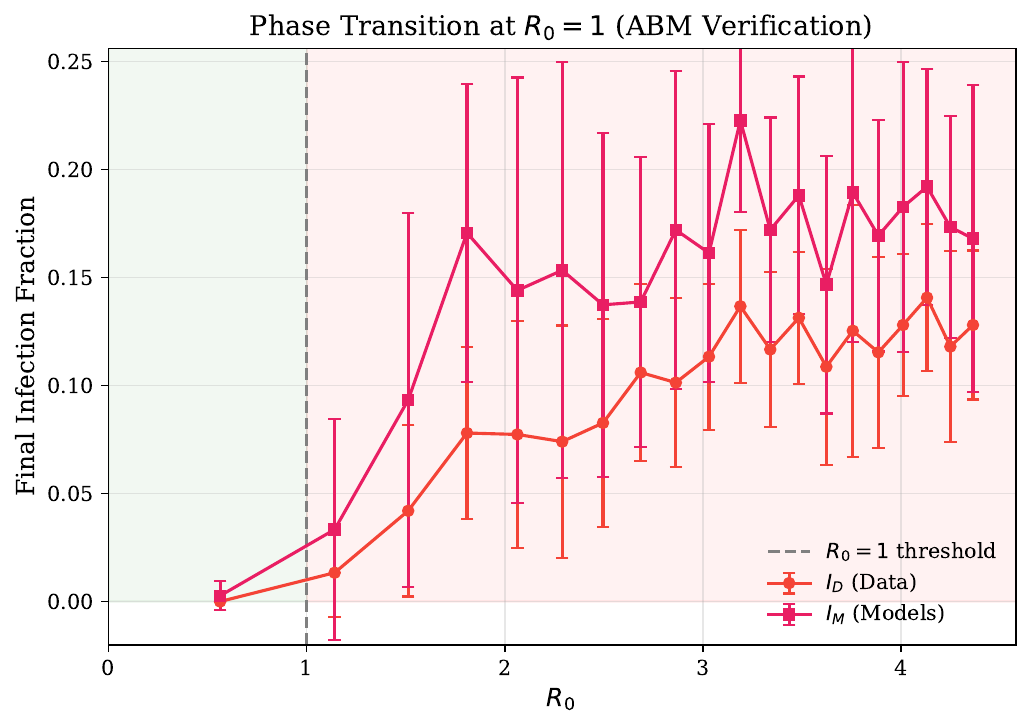}
\caption{ABM threshold verification: sub-critical vs.\ super-critical classification across 20 parameter configurations. 18/20 are correctly identified; the two errors occur near $\Rzero \approx 1$.}
\label{fig:threshold_app}
\end{figure}

\section{Empirical Experiment Details}
\label{app:empirical}

\subsection{Single-Chain Perplexity Trajectories}

\begin{figure}[ht]
\centering
\includegraphics[width=0.8\linewidth]{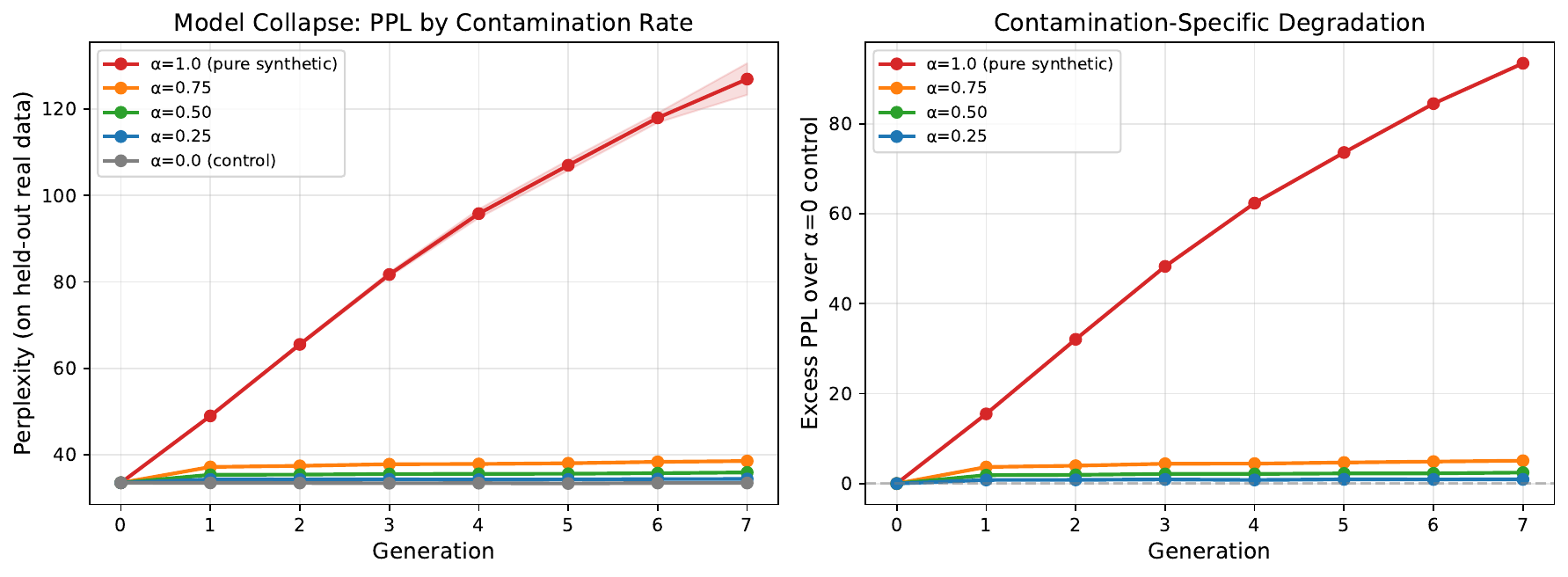}
\caption{Perplexity over 8 generations for WikiText contamination chains at five contamination fractions $\alpha$. Lines show means across 3 seeds; shaded regions are $\pm 1$ std. The $\alpha = 0$ control (gray) remains flat. At $\alpha = 1.0$, perplexity grows from 33.52 to 126.92 ($+93.45$ excess). At $\alpha < 1$, near-plateau dynamics emerge ($r < 0.015$).}
\label{fig:perplexity}
\end{figure}

\subsection{Training Configuration}

\begin{table}[ht]
\centering
\small
\begin{tabular}{@{}ll@{}}
\toprule
\textbf{Hyperparameter} & \textbf{Value} \\
\midrule
Base model & GPT-2 (124M parameters) \\
Training samples per generation & 2000 (WikiText), 1500 (Shakespeare) \\
Max sequence length & 128 tokens \\
Epochs & 3 \\
Batch size & 8 \\
Learning rate & $5 \times 10^{-5}$ \\
Optimizer & AdamW \\
Warmup steps & 50 \\
Weight decay & 0.01 \\
Seeds & $\{42, 123, 456\}$ \\
Evaluation samples & 500 held-out real samples \\
\bottomrule
\end{tabular}
\end{table}

\subsection{Hardware}

All experiments were run on a single NVIDIA RTX 4090 (24GB) GPU.

\subsection{Per-Seed WikiText Results}

\begin{table}[ht]
\centering
\caption{WikiText perplexity per seed at selected generations for $\alpha = 1.0$.}
\small
\begin{tabular}{@{}lccccc@{}}
\toprule
\textbf{Seed} & \textbf{G0} & \textbf{G1} & \textbf{G3} & \textbf{G5} & \textbf{G7} \\
\midrule
42  & 33.53 & 49.08 & 81.21 & 105.83 & 128.23 \\
123 & 33.47 & 48.75 & 81.49 & 106.46 & 121.94 \\
456 & 33.56 & 49.13 & 82.41 & 108.61 & 130.59 \\
\midrule
Mean $\pm$ std & 33.52$\pm$0.04 & 48.99$\pm$0.17 & 81.70$\pm$0.51 & 106.97$\pm$1.19 & 126.92$\pm$3.65 \\
\bottomrule
\end{tabular}
\end{table}

\subsection{Cross-Domain Comparison}

\begin{figure}[ht]
\centering
\includegraphics[width=0.85\linewidth]{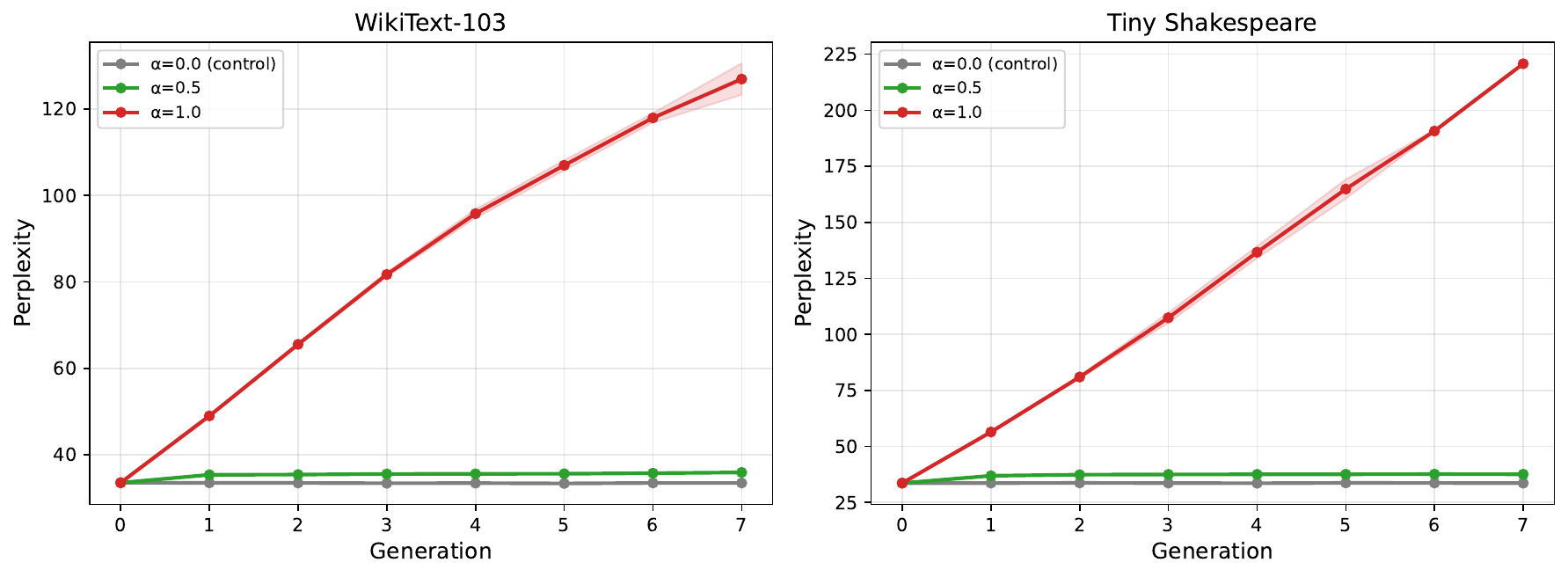}
\caption{Cross-domain comparison of contamination chain dynamics. WikiText (left) and Shakespeare (right) show qualitatively identical patterns: flat control, mild degradation at $\alpha = 0.5$, and strong supercritical growth at $\alpha = 1.0$. Shakespeare exhibits even stronger degradation under pure contamination (6.6$\times$ vs.\ 3.8$\times$), likely due to lower redundancy in literary text.}
\label{fig:cross_domain}
\end{figure}

\subsection{Diversity Collapse Details}

\begin{figure}[ht]
\centering
\includegraphics[width=0.6\linewidth]{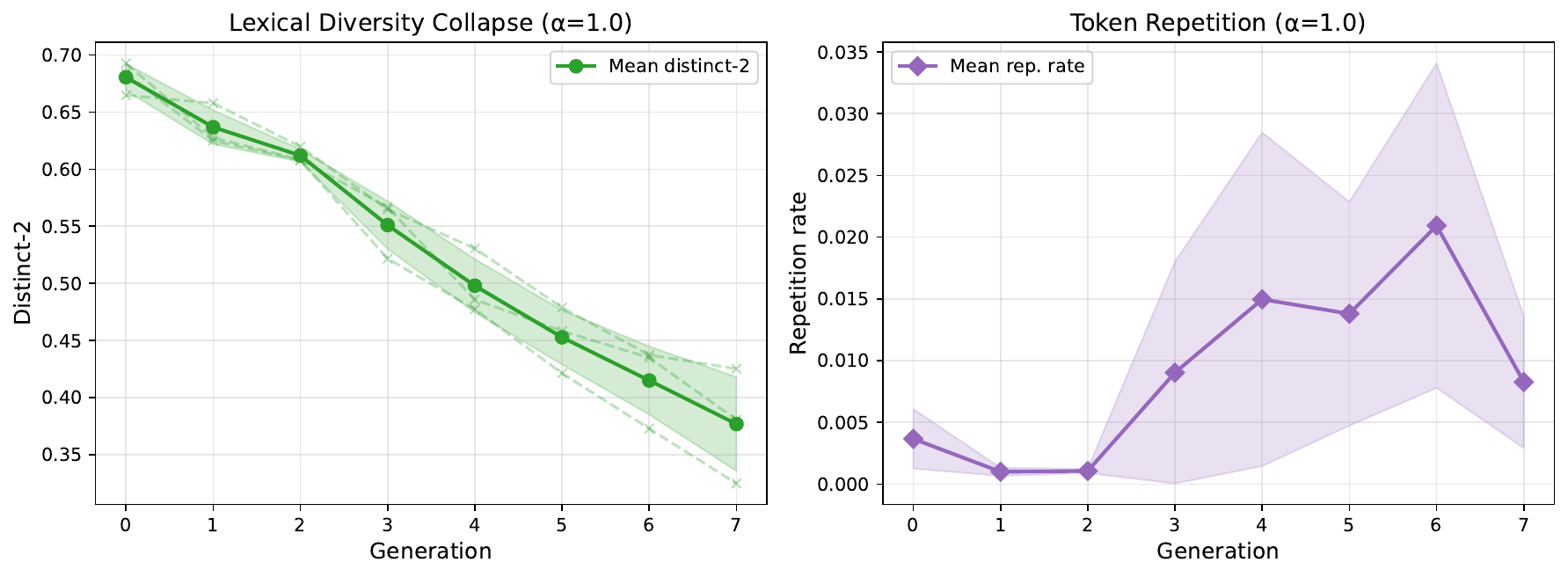}
\caption{Distinct-2 diversity across generations for $\alpha = 1.0$ (WikiText). Diversity drops from 0.68 to 0.38, providing an independent degradation signal correlated with perplexity increase.}
\label{fig:diversity_app}
\end{figure}

\subsection{Convergence Diagnostics}

\begin{figure}[ht]
\centering
\includegraphics[width=0.6\linewidth]{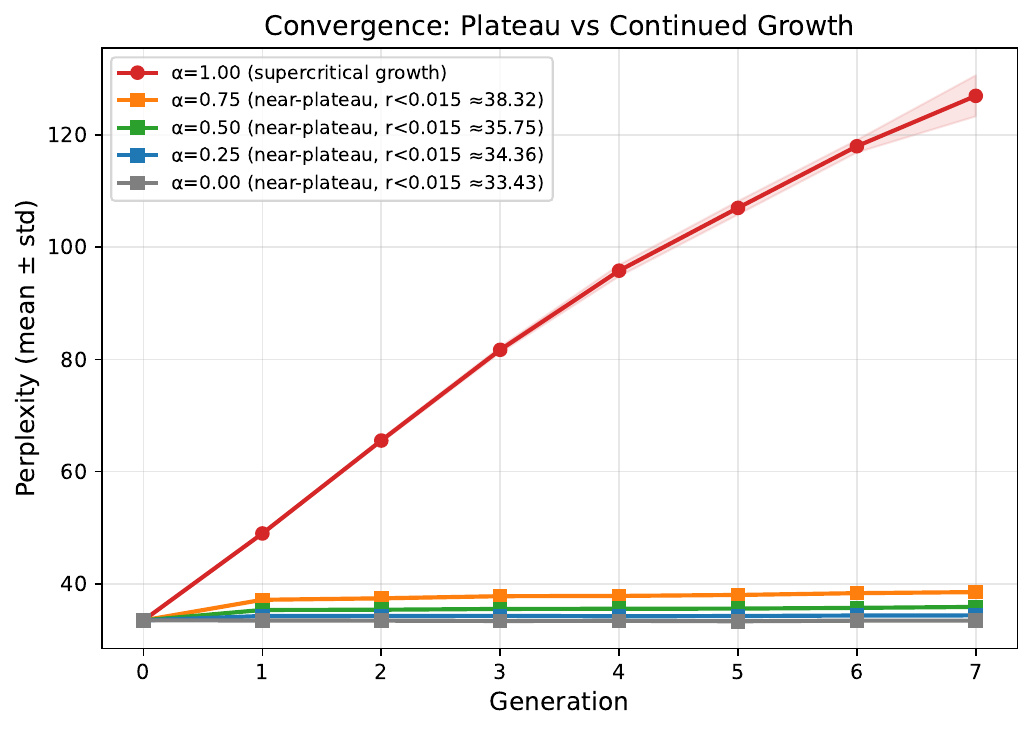}
\caption{Training loss convergence diagnostics across generations and contamination fractions, confirming that all models converge within the 3-epoch training budget.}
\label{fig:convergence_app}
\end{figure}

\subsection{Source-Diversity Experiment Details}
\label{app:multimodel}

The matched-budget source-diversity experiment (\cref{sec:multimodel}) uses GPT-2 (124M) with $K \in \{1, 3, 5\}$ models, 8 seeds $\{42, 123, 456, 789, 1024, 2048, 3141, 4096\}$, and per-model seed offsets $\text{seed} + 1000 \times (m+1)$.

\paragraph{Matched-budget design.}
For each $K$, the synthetic pool has \emph{fixed} size $n{=}2{,}000$: each of $K$ models generates $\lfloor 2000/K \rfloor$ samples, so pool size is identical across conditions.
Each model trains on the full pool at $\alpha{=}1.0$ (pure synthetic).
The $\alpha{=}0$ control uses the same budget with real data only.
All other hyperparameters match the single-chain experiment (\cref{app:empirical}).
Total: $(1 + 1 + 3 + 5) \times 8 \text{ seeds} \times 8 \text{ gens} = 640$ runs.

\paragraph{Per-seed G7 excess PPL.}
\begin{center}
\small
\begin{tabular}{@{}lcccccccc@{}}
\toprule
\textbf{Seed} & 42 & 123 & 456 & 789 & 1024 & 2048 & 3141 & 4096 \\
\midrule
$K{=}1$ & 89.0 & 91.3 & 83.9 & 85.6 & 88.0 & 85.7 & 87.3 & 89.8 \\
$K{=}3$ & 83.7 & 84.6 & 91.9 & 83.4 & 85.5 & 83.8 & 87.6 & 84.0 \\
$K{=}5$ & 89.3 & 86.7 & 85.3 & 83.7 & 86.1 & 82.0 & 87.8 & 84.3 \\
\bottomrule
\end{tabular}
\end{center}

\paragraph{Statistical tests.}
Exact permutation, Wilcoxon, and Jonckheere--Terpstra p-values are one-sided (testing $K_a > K_b$ or decreasing trend); the paired $t$-test p-value is two-sided.
$K{=}1$ vs.\ $K{=}5$: exact permutation $p{=}0.047$ (one-sided), Wilcoxon $p{=}0.055$ (one-sided), paired $t(7){=}2.16$, $p{=}0.068$ (two-sided).
$K{=}1$ vs.\ $K{=}3$: exact permutation $p{=}0.137$, Wilcoxon $p{=}0.125$.
$K{=}3$ vs.\ $K{=}5$: exact permutation $p{=}0.535$ (no difference).
Jonckheere--Terpstra: $Z{=}1.43$, $p{=}0.077$.
The evidence supports a modest buffer effect of $K{>}1$ relative to pure self-training at $\alpha{=}1.0$, not a monotonically increasing benefit with $K$.

\paragraph{Partial contamination ($\alpha{=}0.5$).}
An additional 448 runs test $K \in \{1, 5\}$ at $\alpha{=}0.5$ (same matched-budget design, 8 seeds).
$K{=}1$ excess G7: $+2.20 \pm 0.14$; $K{=}5$: $+2.18 \pm 0.11$.
Mean difference: $+0.02$, paired $t(7) = 0.54$, $p = 0.61$, $d = 0.17$, 3/8 seeds consistent.
The source-diversity effect is undetectable at partial contamination, confirming that real data in the training mix already breaks self-reinforcement.

\section{Intervention Details}
\label{app:interventions}

\subsection{Strategy Summary}

\begin{table}[ht]
\centering
\caption{Intervention strategies and their effects on $\Rzero$ (from baseline $\Rzero = 2.62$). Only watermark-based filtering and herd immunity achieve $\Rzero < 1$ as single strategies.}
\label{tab:interventions}
\small
\begin{tabular}{@{}lllcc@{}}
\toprule
\textbf{Strategy} & \textbf{Mechanism} & \textbf{Parameter} & \textbf{Best $\Rzero$} & $\Rzero < 1$? \\
\midrule
Watermark + Filtering & Detect \& remove synthetic text & $\gammaD \!\uparrow,\, \gammaM \!\uparrow$ & ${<}1$ at 63\% & Yes \\
Herd Immunity & Fraction trained on clean data & $\betaD,\betaM \!\downarrow$ & ${<}1$ at 63\% cov. & Yes \\
Superspreader Control & Limit high-traffic model output & $\betaM \!\downarrow$ & 1.015 & No \\
Open Detection Stds & Standardize detection tools & $\gammaD \!\uparrow,\, \gammaM \!\uparrow$ & 1.142 & No \\
Output Restriction & Rate-limit synthetic generation & $\betaD \!\downarrow$ & 1.172 & No \\
Provenance Tracking & Data lineage metadata & $\gammaD \!\uparrow,\, \betaD \!\downarrow$ & 1.340 & No \\
\bottomrule
\end{tabular}
\end{table}

The herd immunity threshold is $1 - 1/\Rzero = 61.9\%$ at baseline---the ecosystem analog of vaccine coverage thresholds.

\subsection{Full Intervention Curves}

\begin{figure}[ht]
\centering
\includegraphics[width=0.7\linewidth]{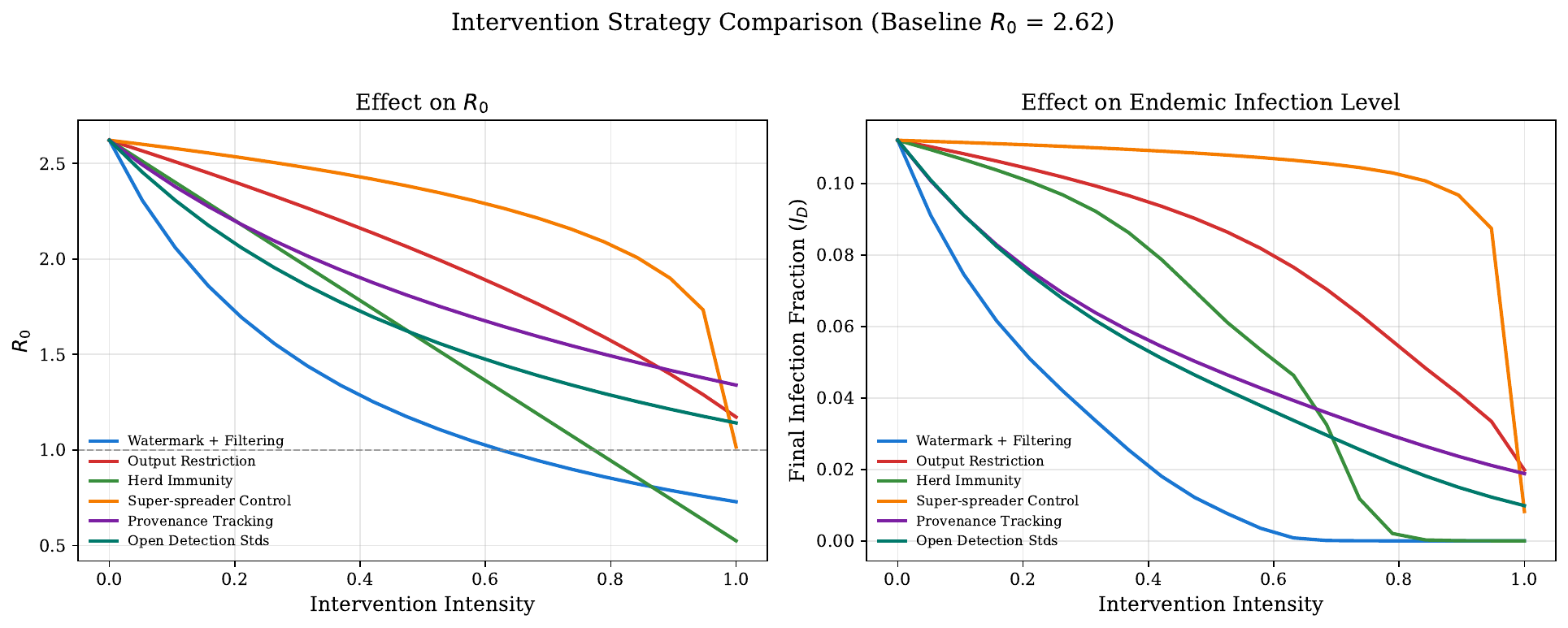}
\caption{$\Rzero$ as a function of intervention intensity for all six strategies. Only watermark-based filtering and herd immunity cross the $\Rzero = 1$ threshold.}
\label{fig:interventions_app}
\end{figure}

\begin{figure}[ht]
\centering
\includegraphics[width=0.7\linewidth]{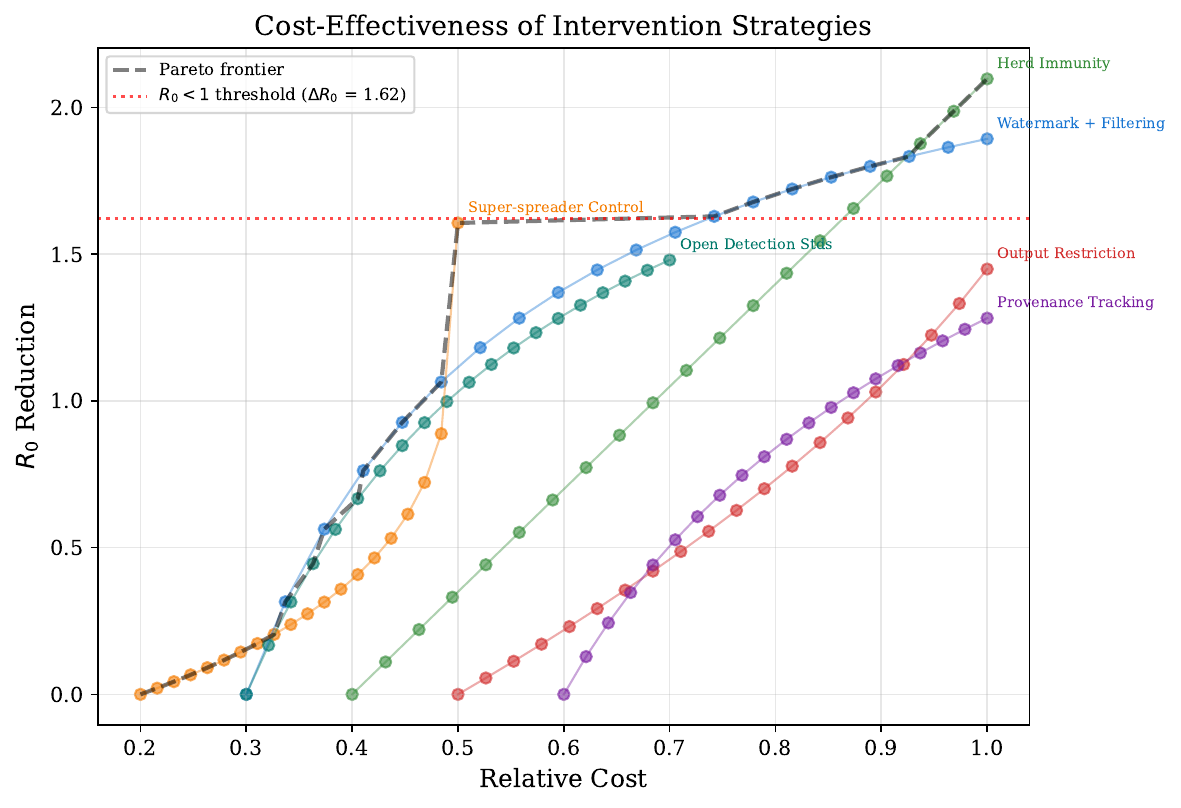}
\caption{Pareto frontier: illustrative cost vs.\ $\Rzero$ reduction across single-strategy sweeps (6 strategies $\times$ 20 intensity levels $=$ 120 points). Only watermark-based filtering and herd immunity achieve $\Rzero < 1$ alone. Combined interventions (15 pairs $\times$ 9 intensity combinations $=$ 135 evaluations) are analyzed separately in \cref{sec:interventions}.}
\label{fig:pareto}
\end{figure}

\subsection{Calibration Scenario Trajectories}

\begin{figure}[ht]
\centering
\includegraphics[width=0.7\linewidth]{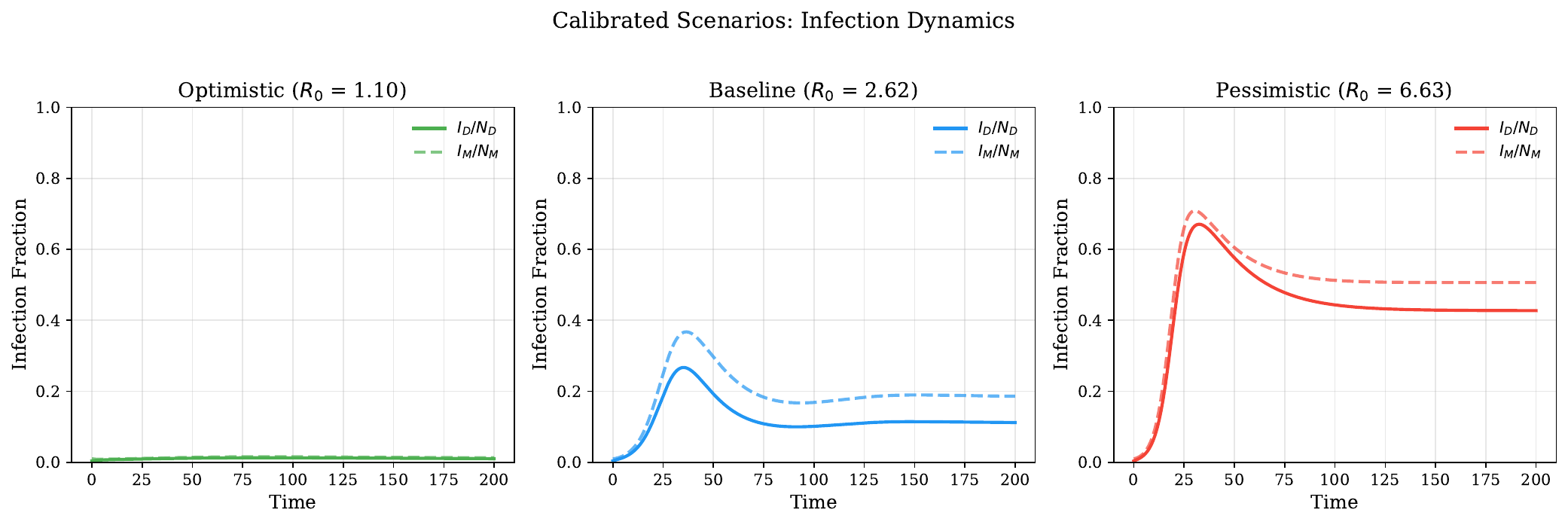}
\caption{ODE infection trajectories under the three calibration scenarios. The pessimistic scenario ($\Rzero = 6.63$) converges rapidly to a high endemic level; the optimistic scenario ($\Rzero = 1.10$) shows slow, marginal growth.}
\label{fig:scenarios_app}
\end{figure}

\subsection{Bifurcation Diagram}

\begin{figure}[ht]
\centering
\includegraphics[width=0.6\linewidth]{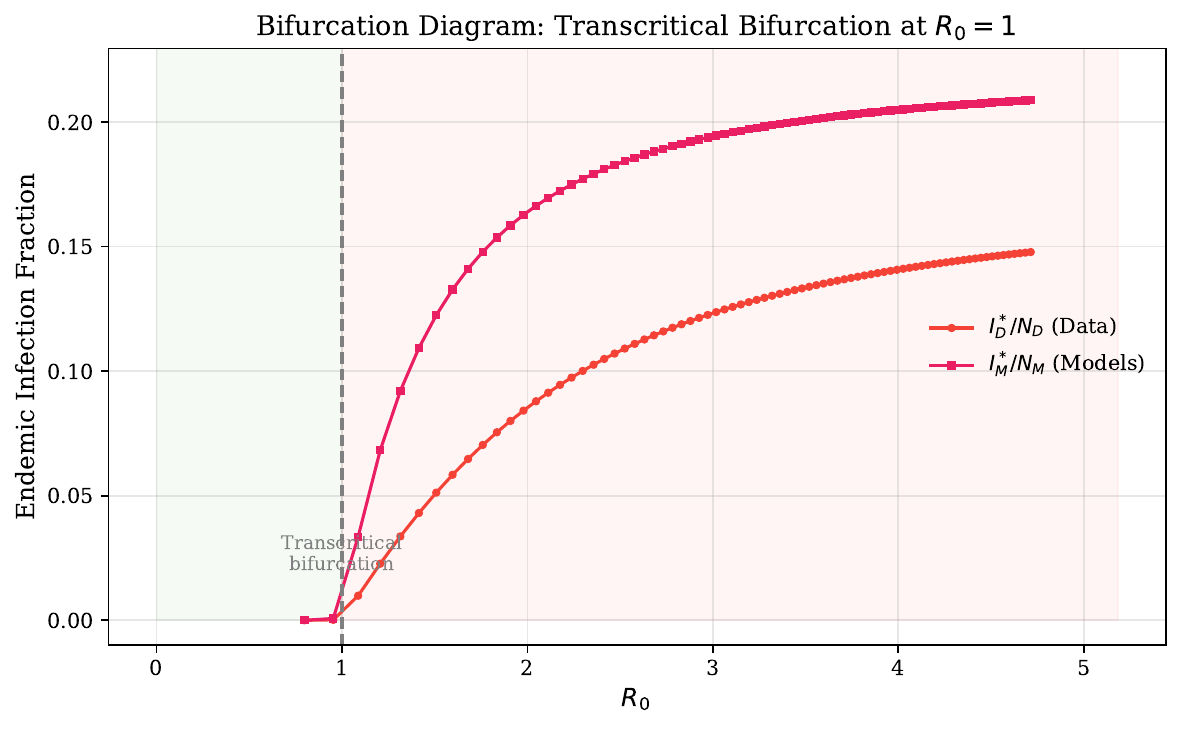}
\caption{Bifurcation diagram showing endemic equilibrium infection level as a function of $\Rzero$. The transcritical bifurcation at $\Rzero = 1$ is clearly visible: the DFE (zero infection) is stable for $\Rzero < 1$ and unstable for $\Rzero > 1$, where the endemic branch emerges.}
\label{fig:bifurcation_app}
\end{figure}

\section{Code--Paper Theorem Numbering}
\label{app:theorem_numbering}

The verification code (\texttt{theorem\_verification.json}) uses a different numbering convention than the paper.
The mapping is:
\begin{itemize}[leftmargin=*,itemsep=1pt]
\item \texttt{theorem1\_dfe\_exists} $\to$ existence part of \cref{thm:dfe} (Theorem~2).
\item \texttt{theorem2\_dfe\_stability} $\to$ stability part of \cref{thm:dfe}.
\item \texttt{theorem3\_ee\_existence} $\to$ \cref{thm:ee} (Proposition~3).
\item \texttt{theorem4\_bifurcation} $\to$ \cref{thm:bifurcation} (Proposition~4).
\item \texttt{proposition5\_sirs\_oscillation} $\to$ \cref{prop:sirs} (Proposition~5).
\end{itemize}
The $\Rzero$ formula (\cref{thm:R0}, Theorem~1 in the paper) is verified by direct comparison of the analytic formula with the NGM spectral radius, not as a separate entry in the verification JSON.

\section{Reproducibility Statement}

All experiments use publicly available models (GPT-2 from HuggingFace) and datasets (WikiText-103, Tiny Shakespeare).
Random seeds are fixed at $\{42, 123, 456\}$ for single-chain experiments and $\{42, 123, 456, 789, 1024, 2048, 3141, 4096\}$ for source-diversity experiments.
The ODE solver uses SciPy's \texttt{solve\_ivp} with RK45, relative tolerance $10^{-8}$, and absolute tolerance $10^{-10}$.
The ABM is implemented in Python with NetworkX.
Theorem verification counts (e.g., 169/200 for EE existence, 196/200 for bifurcation) are from a single run of the verification script with unseeded random parameter draws; exact counts may vary on rerun, though the qualitative result (near-100\% agreement) is stable.
Code will be released upon publication.
Total compute: approximately 5 GPU-hours (single-chain, 192 runs) $+$ 12 GPU-hours ($K$-sweep, 640 runs) $+$ 8 GPU-hours ($\alpha{=}0.5$ robustness, 448 runs) $\approx$ 25 GPU-hours on a single NVIDIA RTX 4090.